\documentclass[twoside,11pt]{article}

\usepackage{jmlr2e,bm}
\begin{document}

\title{PAC-Bayes Analysis of Multi-view Learning}

\author{\name Shiliang Sun \email shiliangsun@gmail.com \\
       \addr Shanghai Key Laboratory of Multidimensional Information
       Processing\\
       Department of Computer Science and Technology\\East China
 Normal University\\ 500 Dongchuan Road, Shanghai 200241, China
       \AND
       \name John Shawe-Taylor \email j.shawe-taylor@ucl.ac.uk \\
       \addr Department of Computer Science\\
University College London\\ Gower Street, London WC1E 6BT, United
Kingdom
       \AND
       \name Liang Mao \email lmao14@outlook.com \\
              \addr Shanghai Key Laboratory of Multidimensional Information
              Processing\\
              Department of Computer Science and Technology\\East China
        Normal University\\ 500 Dongchuan Road, Shanghai 200241, China
}

\editor{}

\maketitle

\begin{abstract}%   <- trailing '%' for backward compatibility of .sty file
This paper presents eight PAC-Bayes bounds to analyze the
generalization performance of multi-view classifiers. These bounds
adopt data dependent Gaussian priors which emphasize classifiers
with high view agreements. The center of the prior for the first two
bounds is the origin, while the center of the prior for the third
and fourth bounds is given by a data dependent vector. An
important technique to obtain these bounds is two derived
logarithmic determinant inequalities whose difference lies in
whether the dimensionality of data is involved. The centers of the
fifth and sixth bounds are calculated on a separate subset of the training
set. The last two bounds
use unlabeled data to represent view agreements and are thus
applicable to semi-supervised multi-view learning. We evaluate all
the presented multi-view PAC-Bayes bounds on benchmark data and
compare them with previous single-view PAC-Bayes bounds. The
usefulness and performance of the multi-view bounds are discussed.
\end{abstract}


\setlength{\arraycolsep}{0.1em}
\section{Introduction}

Multi-view learning is a promising research direction with prevalent
applicability~\citep{Sun13survey}. For instance, in multimedia
content understanding, multimedia segments can be described by both
their video and audio signals, and the video and audio signals are
regarded as the two views. Learning from data relies on collecting data that contain a
sufficient signal and encoding our prior knowledge in increasingly
sophisticated regularization schemes that enable the signal to be
extracted. With certain co-regularization schemes, multi-view learning
performs well on various learning tasks.

Statistical learning theory (SLT) provides a general framework to
analyze the generalization performance of machine learning
algorithms. The theoretical outcomes can be used to motivate
algorithm design, select models or give insights on the effects and
behaviors of some interesting quantities. For example, the
well-known large margin principle in support vector machines (SVMs)
is well supported by various SLT
bounds~\citep{Vapnik98SLT,Bartlett02,Sun10JMLR}. Different from
early bounds that often rely on the complexity measures of the
considered function classes, the recent PAC-Bayes
bounds~\citep{McAllester99,Seeger02,Langford05} give the tightest
predictions of the generalization performance, for which the prior
and posterior distributions of learners are involved on top of the
PAC (Probably Approximately Correct) learning
setting~\citep{Catoni07,Germain09}. Beyond the common supervised
learning, PAC-Bayes analysis has also been applied to other tasks,
e.g., density estimation~\citep{Seldin10,Higgs10} and reinforcement
learning~\citep{Seldin12}.

Although the field of
multi-view learning has enjoyed a great success with algorithms and
applications and is provided with some theoretical results,
PAC-Bayes analysis of multi-view learning is still absent. In this paper,
we attempt to fill the gap between the developments in theory and
practice by proposing new PAC-Bayes bounds for multi-view learning.

An earlier attempt to analyze the generalization of two-view
learning was made using Rademacher
complexity~\citep{Farquhar06,Rosenberg07}. The bound relied on
estimating the empirical Rademacher complexity of the class of pairs
of functions from the two views that are matched in expectation
under the data generating distribution. Hence, this approach also
implicitly relied on the data generating distribution to define the
function class (and hence prior). The current paper makes the
definition of the prior in terms of the data generating distribution
explicit through the PAC-Bayes framework and provides several bounds. However, the
main advantage is that it defines a framework that makes explicit
the definition of the prior in terms of the data generating
distribution, setting a template for other related approaches to
encoding complex prior knowledge that relies on the data generating
distribution.

\citet{Kakade07} characterized the expected regret of a
semi-supervised multi-view regression algorithm. The
results given by~\citet{Sridharan08} take an information theoretic
approach that involves a number of assumptions that may be difficult
to check in practice. With these assumptions theoretical results
including PAC-style analysis to bound expected losses were given,
which involve some Bayes optimal predictor and but cannot provide
computable classification error bounds since the data generating
distribution is usually unknown. These results therefore represent a
related but distinct approach.

We adopt a PAC-Bayes analysis where we encode our assumptions
through priors defined in terms of the data generating distribution.
Such priors have been studied by~\citet{Catoni07} under the name of
localized priors and more recently by~\citet{Lever13TPB} as data
distribution dependent priors. Both papers considered schemes for
placing a prior over classifiers defined through their true
generalization errors. In contrast, the prior that we consider is
mainly used to encode the assumption about the relationship between
the two views in the data generating distribution. Such data
distribution dependent priors cannot be subjected to traditional
Bayesian analysis since we do not have an explicit form for the
prior, making inference impossible. Hence, this paper illustrates
one of the advantages that arise from the PAC-Bayes framework.

The PAC-Bayes theorem bounds the true error of the distribution of
classifiers in terms of a term from the sample complexity and the
KL divergence between the posterior and the prior distributions of classifiers.
The key technical innovations of the paper enable the bounding of
the KL divergence term in terms of empirical quantities despite
involving priors that cannot be computed. This approach was adopted
in \citet{EmilioJMLR12} for some simple priors such as the Gaussian
centered at $\mathbb{E}[y\phi(\mathbf{x})]$. The current paper
treats a significantly more sophisticated case where the priors
encode our expectation that good weight vectors can be found that
give similar outputs from both views.

Specifically, we first provide four PAC-Bayes bounds using priors
that reflect how well the two views agree on average over all
examples. The first two bounds use a Gaussian prior centered at the
origin, while the third and fourth ones adopt a different prior
whose center is not the origin. However, the formulations of the
priors involve mathematical expectations with respect to the unknown
data distributions. We manage to bound the expectation related terms
with their empirical estimations on a finite sample of data. Then,
we further provide two PAC-Bayes bounds using a part of the training data
to determine priors, and two PAC-Bayes bounds for semi-supervised
multi-view learning where unlabeled data are involved in the
definition of the priors.

When a natural feature split does not exist, multi-view learning
could still obtain performance improvements with manufactured splits,
provided that each of the views contains not only enough information
for the learning task itself, but some knowledge that other views do not have.
It is therefore important that people should split features into views satisfying
the assumptions. However, data split is still an open question and beyond the
scope of this paper.

The rest of this paper is organized as follows. After briefly
reviewing the PAC-Bayes bound for SVMs in Section~\ref{prePACB}, we
give and derive four multi-view PAC-Bayes bounds involving only
empirical quantities in Section~\ref{First2bounds} and
Section~\ref{Last2bounds}. Then we give two bounds whose centers are
calculated on a separate subset of the training data in Section~\ref{stdbounds}. After that, we present two semi-supervised
multi-view PAC-Bayes bounds in Section~\ref{secSSLBs}. The
optimization formulations of the related single-view and multi-view
SVMs as well as semi-supervised multi-view SVMs are given in
Section~\ref{Algorithms}. After evaluating the usefulness and
performance of the bounds in Section~\ref{expriments}, we give
concluding remarks in Section~\ref{conclusion}.

\section{PAC-Bayes Bound and Specialization to SVMs}
\label{prePACB}

Consider a binary classification problem. Let $\mathcal{D}$ be the
distribution of feature $\mathbf{x}$ lying in an input space
$\mathcal{X}$ and the corresponding output label $y$ where $y \in
\{-1,1\}$. Suppose $Q$ is a posterior distribution over the parameters of the
classifier $c$. Define the true error and empirical error of a
classifier as
\begin{eqnarray}
e_{\mathcal{D}} &=&Pr_{(\mathbf{x},y)\sim
\mathcal{D}}(c(\mathbf{x})\neq y),\nonumber\\
\hat{e}_S&=&Pr_{(\mathbf{x},y)\sim S} (c(\mathbf{x})\neq
y)=\frac{1}{m} \sum_{i=1}^m I(c(\mathbf{x}_i)\neq y_i),\nonumber
\end{eqnarray}
where $S$ is a sample  including $m$ examples, and $I(\cdot)$ is the
indicator function. With the distribution $Q$, we can then define
the average true error $E_{Q,\mathcal{D}}=\mathbb{E}_{c\sim Q}
e_{\mathcal{D}}$, and the average empirical error
$\hat{E}_{Q,S}=\mathbb{E}_{c\sim Q} \hat{e}_S$. The following lemma
provides the PAC-Bayes bound on $E_{Q,\mathcal{D}}$ in the current context
of binary classification.

\begin{theorem}[PAC-Bayes Bound~\citep{Langford05}]
For any data distribution $\mathcal {D}$, for any prior P(c) over
the classifier $c$, for any $\delta\in(0,1]$:
\begin{equation}
Pr_{S\sim \mathcal{D}^m}\left(\forall
Q(c):KL_+(\hat{E}_{Q,S}||E_{Q,\mathcal{D}})\leq
\frac{KL(Q||P)+\ln(\frac{m+1}{\delta})}{m}\right)\geq 1-\delta,
\nonumber
\end{equation}
where $KL(Q||P)=\mathbb{E}_{c\sim Q}\ln\frac{Q(c)}{P(c)}$ is the KL
divergence between $Q$ and $P$, and
$KL_+(q||p)=q\ln\frac{q}{p}+(1-q)\ln\frac{1-q}{1-p}$ for $p>q$ and
$0$ otherwise. \label{LemLangford}
\end{theorem}

Suppose from the $m$ training examples we learn an SVM classifier
represented by $c_{\mathbf{u}}(\mathbf{x})=\mbox{sign}
(\mathbf{u}^\top\bm \phi(\mathbf{x}))$, where $\phi(\mathbf{x})$ is
a projection of the original feature to a certain feature space
induced by some kernel function. Define the prior and the posterior
of the classifier to be Gaussian with $\mathbf{u}\sim
\mathcal{N}(\mathbf{0},\mathbf{I})$ and  $\mathbf{u}\sim
\mathcal{N}(\mu\mathbf{w},\mathbf{I})$, respectively. Note that here
$\|\mathbf{w}\|=1$, and thus the distance between the center of the
posterior and the origin is $\mu$. With this specialization, we give
the PAC-Bayes bound for SVMs~\citep{Langford05,EmilioJMLR12} below.

\begin{theorem}
     For any data distribution $\mathcal{D}$, for
     any $\delta\in(0,1]$, we have
\begin{equation}
Pr_{S\sim \mathcal {D}^m}\left(\forall \mathbf{w}, \mu:
KL_+(\hat{E}_{Q,S}(\mathbf{w}, \mu)||E_{Q,\mathcal{D}}(\mathbf{w},
\mu))\leq
\frac{\frac{\mu^2}{2}+\ln(\frac{m+1}{\delta})}{m}\right)\geq
1-\delta, \nonumber
\end{equation}
where $\|\mathbf{w}\|=1$. \label{thmStaPB}
\end{theorem}

All that remains is calculating the empirical stochastic error rate
$\hat{E}_{Q,S}$. It can be shown that for a posterior $Q=\mathcal{N}(\mu
\mathbf{w}, \mathbf{I})$ with $\|\mathbf{w}\|=1$, we have
\begin{equation}
\hat{E}_{Q,S}=\mathbb{E}_S \left[\tilde{F}(\mu \gamma
(\mathbf{x},y))\right] , \nonumber
\end{equation}
where $\mathbb{E}_S$ is the average over the $m$ training examples,
$\gamma (\mathbf{x},y)$ is the normalized margin of the example
\begin{equation}
\gamma (\mathbf{x},y)=y \mathbf{w}^\top
\phi(\mathbf{x})/\|\phi(\mathbf{x})\| , \nonumber
\end{equation}
and $\tilde{F}(x)$ is the Gaussian cumulative distribution
\begin{equation}
\tilde{F}(x)=\int_x^\infty \frac{1}{\sqrt{2\pi}}e^{-x^2/2} d x.
\nonumber
\end{equation}

The generalization error of the original SVM classifier
$c_{\mathbf{w}}(\mathbf{x})=\mbox{sign} (\mathbf{w}^\top\bm
\phi(\mathbf{x}))$ can be bounded by at most twice the
average true error $E_{Q,\mathcal{D}}(\mathbf{w}, \mu)$ of the
corresponding stochastic classifier~\citep{Langford02nips}. That is,
for any $\mu$ we have
\begin{equation}
Pr_{(\mathbf{x},y)\sim \mathcal{D}} \left(\mbox{sign}
(\mathbf{w}^\top\bm \phi(\mathbf{x}))\neq y\right) \leq 2
E_{Q,\mathcal{D}}(\mathbf{w}, \mu) . \nonumber
\end{equation}

\section{Multi-view PAC-Bayes Bounds}
\label{First2bounds}

We propose a new data dependent prior for PAC-Bayes analysis of
multi-view learning. In particular, we take the distribution on the
concatenation of the two weight vectors $\mathbf{u}_1$ and
$\mathbf{u}_2$ as their individual product:
$\tilde{P}([\mathbf{u}_1^\top, \mathbf{u}_2^\top]^\top) =
P_1(\mathbf{u}_1)P_2(\mathbf{u}_2)$ but then weight it in some
manner associated with how well the two weights agree averagely on
all examples. That is, the prior is
$$P([\mathbf{u}_1^\top,
\mathbf{u}_2^\top]^\top) \propto P_1(\mathbf{u}_1)P_2(\mathbf{u}_2)
V(\mathbf{u}_1,\mathbf{u}_2),$$ where $P_1(\mathbf{u}_1)$ and
$P_1(\mathbf{u}_2)$ are Gaussian with zero mean  and identity
covariance, and
$$V(\mathbf{u}_1,\mathbf{u}_2) =\exp\left\{-
\frac{1}{2\sigma^2}\mathbb{E}_{(\mathbf{x}_1,\mathbf{x}_2)}(\mathbf{x}_1^\top
\mathbf{u}_1 - \mathbf{x}_2^\top \mathbf{u}_2)^2\right\} .$$

To specialize the PAC-Bayes bound for multi-view learning, we
consider classifiers of the form
\begin{equation}
c(\mathbf{x}) = \mbox{sign} (\mathbf{u}^\top \phi(\mathbf{x})) ,
\nonumber
\end{equation}
where $\mathbf{u}=[\mathbf{u}_1^\top, \mathbf{u}_2^\top]^\top$ is
the concatenated weight vector from two views, and
$\phi(\mathbf{x})$ can be the concatenated
$\mathbf{x}=[\mathbf{x}_1^\top, \mathbf{x}_2^\top]^\top$ itself or a
concatenation of maps of $\mathbf{x}$ to kernel-induced feature
spaces. Note that $\mathbf{x}_1$ and $\mathbf{x}_2$ indicate
features of one example from the two views, respectively. For
simplicity, here we use the original features to derive our results,
though kernel maps can be implicitly employed as well. Our dimensionality
independent bounds work even when the dimension of the kernelized feature space
goes to infinity.

According to our setting, the classifier prior is fixed to be
\begin{equation}
P(\mathbf{u}) \propto \mathcal{N}(\mathbf{0}, \mathbf{I})\times
V(\mathbf{u}_1,\mathbf{u}_2), \label{prior6}
\end{equation}
Function $ V(\mathbf{u}_1,\mathbf{u}_2) $ makes the prior place large
probability mass on parameters with which the classifiers from two views
agree well on all examples averagely.
The posterior is chosen to be of the form
\begin{equation}
Q(\mathbf{u}) = \mathcal{N}(\mu \mathbf{w},
\mathbf{I}),\label{prior7}
\end{equation}
where $\|\mathbf{w}\|=1$.

Define $\mathbf{\tilde{x}}=[\mathbf{x}_1^\top,
-\mathbf{x}_2^\top]^\top$. We have
\begin{eqnarray}
P(\mathbf{u}) &\propto& \mathcal{N}(\mathbf{0}, \mathbf{I})\times
V(\mathbf{u}_1,\mathbf{u}_2) \nonumber\\
&\propto& \exp\left\{-\frac{1}{2} \mathbf{u}^\top \mathbf{u}\right\}
\times \exp\left\{-
\frac{1}{2\sigma^2}\mathbb{E}_{(\mathbf{x}_1,\mathbf{x}_2)}(\mathbf{x}_1^\top
\mathbf{u}_1 - \mathbf{x}_2^\top \mathbf{u}_2)^2\right\} \nonumber\\
&=& \exp\left\{-\frac{1}{2} \mathbf{u}^\top \mathbf{u}\right\}
\times \exp\left\{-
\frac{1}{2\sigma^2}\mathbb{E}_{\mathbf{\tilde{x}}}(\mathbf{u}^\top
\mathbf{\tilde{x}} \mathbf{\tilde{x}}^\top \mathbf{u}) \right\} \nonumber\\
&=& \exp\left\{-\frac{1}{2} \mathbf{u}^\top \mathbf{u}\right\}
\times \exp\left\{- \frac{1}{2\sigma^2}\mathbf{u}^\top
\mathbb{E}(\mathbf{\tilde{x}} \mathbf{\tilde{x}}^\top) \mathbf{u}
\right\}\nonumber\\
&=& \exp\left\{-\frac{1}{2} \mathbf{u}^\top \left(\mathbf{I} +
\frac{\mathbb{E}(\mathbf{\tilde{x}}
\mathbf{\tilde{x}}^\top)}{\sigma^2}\right) \mathbf{u}\right\}.
\nonumber
\end{eqnarray}
That is, $P(\mathbf{u})=\mathcal{N}(\mathbf{0}, \Sigma)$ with
$\Sigma=\left(\mathbf{I} + \frac{\mathbb{E}(\mathbf{\tilde{x}}
\mathbf{\tilde{x}}^\top)}{\sigma^2}\right)^{-1}$.

Suppose $\dim(\mathbf{u})=d$. Given the above prior and posterior,
we have the following theorem to characterize their divergence.
\begin{theorem}
\begin{equation}
KL(Q(\mathbf{u})\|P(\mathbf{u}))= \frac{1}{2}\left(-\ln
(\Big|\mathbf{I} + \frac{\mathbb{E}(\mathbf{\tilde{x}}
\mathbf{\tilde{x}}^\top)}{\sigma^2}\Big|) + \frac{1}{\sigma^2}
\mathbb{E} [\mathbf{\tilde{x}}^\top \mathbf{\tilde{x}} + \mu^2
(\mathbf{w}^\top\mathbf{\tilde{x}})^2]+ \mu^2\right)
.\label{eqnKLpq}
\end{equation}
\label{thmKLpq}
\end{theorem}
\begin{proof}
It is easy to show that the KL divergence between two
Gaussians~\citep{Rasmussen06GP} in an $N$-dimensional space is
\begin{equation}
KL(\mathcal{N}(\bm\mu_0,\Sigma_0)\|\mathcal{N}(\bm\mu_1,\Sigma_1))
=\frac{1}{2}\left(\ln(\frac{\big|\Sigma_1\big|}{\big|\Sigma_0\big|})+\mbox{tr}(\Sigma_1^{-1}\Sigma_0)+
(\bm\mu_1-\bm\mu_0)^\top\Sigma_1^{-1}(\bm\mu_1-\bm\mu_0) -d \right).
\nonumber
\end{equation}

The KL divergence between the
posterior and prior is thus
\begin{eqnarray}
KL(Q(\mathbf{u})\|P(\mathbf{u}))&=&\frac{1}{2}\left(-\ln
(\Big|\mathbf{I} + \frac{\mathbb{E}(\mathbf{\tilde{x}}
\mathbf{\tilde{x}}^\top)}{\sigma^2}\Big|) + \mbox{tr}(\mathbf{I} +
\frac{\mathbb{E}(\mathbf{\tilde{x}}
\mathbf{\tilde{x}}^\top)}{\sigma^2}) + \mu^2 \mathbf{w}^\top
(\mathbf{I} + \frac{\mathbb{E}(\mathbf{\tilde{x}}
\mathbf{\tilde{x}}^\top)}{\sigma^2}) \mathbf{w} - d \right) \nonumber\\
&=& \frac{1}{2}\left(-\ln (\Big|\mathbf{I} +
\frac{\mathbb{E}(\mathbf{\tilde{x}}
\mathbf{\tilde{x}}^\top)}{\sigma^2}\Big|) + \mbox{tr}(
\frac{\mathbb{E}(\mathbf{\tilde{x}}
\mathbf{\tilde{x}}^\top)}{\sigma^2}) + \mu^2 \mathbf{w}^\top
(\frac{\mathbb{E}(\mathbf{\tilde{x}}
\mathbf{\tilde{x}}^\top)}{\sigma^2}) \mathbf{w} + \mu^2\right) \nonumber\\
&=& \frac{1}{2}\left(-\ln (\Big|\mathbf{I} +
\frac{\mathbb{E}(\mathbf{\tilde{x}}
\mathbf{\tilde{x}}^\top)}{\sigma^2}\Big|) + \frac{1}{\sigma^2}
\mathbb{E}[\mbox{tr}(\mathbf{\tilde{x}} \mathbf{\tilde{x}}^\top)] +
\frac{\mu^2}{\sigma^2}
\mathbb{E}[(\mathbf{w}^\top\mathbf{\tilde{x}})^2]
+ \mu^2\right) \nonumber\\
&=& \frac{1}{2}\left(-\ln (\Big|\mathbf{I} +
\frac{\mathbb{E}(\mathbf{\tilde{x}}
\mathbf{\tilde{x}}^\top)}{\sigma^2}\Big|) + \frac{1}{\sigma^2}
\mathbb{E}[\mathbf{\tilde{x}}^\top \mathbf{\tilde{x}}] +
\frac{\mu^2}{\sigma^2}
\mathbb{E}[(\mathbf{w}^\top\mathbf{\tilde{x}})^2] +
\mu^2\right)\nonumber\\
&=& \frac{1}{2}\left(-\ln (\Big|\mathbf{I} +
\frac{\mathbb{E}(\mathbf{\tilde{x}}
\mathbf{\tilde{x}}^\top)}{\sigma^2}\Big|) + \frac{1}{\sigma^2}
\mathbb{E} [\mathbf{\tilde{x}}^\top \mathbf{\tilde{x}} + \mu^2
(\mathbf{w}^\top\mathbf{\tilde{x}})^2]+ \mu^2\right),\nonumber
\end{eqnarray}
which completes the proof.
\end{proof}

The problem with this expression is that it contains expectations
over the input distribution that we are unable to compute. This is
because we have defined the prior distribution in terms of the input
distribution via the $V$ function. Such priors are referred to as
localized by~\citet{Catoni07}. While his work considered specific
examples of such priors that satisfy certain optimality conditions,
the definition we consider here is encoding natural prior
assumptions about the link between the input distribution and the
classification function, namely that it will have a simple
representation in both views. This is an example of
luckiness~\citep{JST98}, where generalization is estimated making
assumptions that if proven true lead to tighter bounds, as for
example in the case of a large margin classifier.

We now develop methods that estimate the relevant quantities
in~(\ref{eqnKLpq}) from empirical data, so that there will be
additional empirical estimations involved in the final bounds besides
the usual empirical error.

We proceed to provide and prove two inequalities on the involved
logarithmic determinant function, which are very important for the
subsequent multi-view PAC-Bayes bounds.

\begin{theorem}
\begin{eqnarray}
-\ln \Big|\mathbf{I} + \frac{\mathbb{E}(\mathbf{\tilde{x}}
\mathbf{\tilde{x}}^\top)}{\sigma^2}\Big| &\leq& -d \ln
\mathbb{E}\Big[\Big|\mathbf{I} + \frac{\mathbf{\tilde{x}}
\mathbf{\tilde{x}}^\top}{\sigma^2}\Big|^{1/d}\Big], \label{eqn11022}\\
-\ln \Big|\mathbf{I} + \frac{\mathbb{E}(\mathbf{\tilde{x}}
\mathbf{\tilde{x}}^\top)}{\sigma^2}\Big| &\leq&
-\mathbb{E}\ln\Big|\mathbf{I} + \frac{\mathbf{\tilde{x}}
\mathbf{\tilde{x}}^\top}{\sigma^2}\Big|. \label{eqn11023}
\end{eqnarray}
\end{theorem}
\begin{proof}
According to the Minkowski determinant theorem, for $n\times n$
positive semi-definite matrices $A$ and $B$, the following
inequality holds
\begin{equation}
\big|A+B\big|^{1/n}\geq \big| A\big|^{1/n} + \big|B\big|^{1/n} , \nonumber
\end{equation}
which implies that the function $A\mapsto\big|A\big|^{1/n}$ is concave
on the set of $n\times n$ positive semi-definite matrices.
Therefore, with Jensen's inequality we have
\begin{eqnarray}
-\ln \Big|\mathbf{I} + \frac{\mathbb{E}(\mathbf{\tilde{x}}
\mathbf{\tilde{x}}^\top)}{\sigma^2}\Big|&=& -d \ln
\Big|\mathbb{E}(\mathbf{I} + \frac{\mathbf{\tilde{x}}
\mathbf{\tilde{x}}^\top}{\sigma^2})\Big|^{1/d} \nonumber\\
&\leq& -d \ln \mathbb{E}\Big[\Big|\mathbf{I} +
\frac{\mathbf{\tilde{x}}
\mathbf{\tilde{x}}^\top}{\sigma^2}\Big|^{1/d}\Big] . \nonumber
\end{eqnarray}
Since the natural logarithm is concave, we further have
\begin{eqnarray}
-d \ln \mathbb{E}\Big[\Big|\mathbf{I} + \frac{\mathbf{\tilde{x}}
\mathbf{\tilde{x}}^\top}{\sigma^2}\Big|^{1/d}\Big]&\leq& -d
\mathbb{E}\Big[\ln\Big|\mathbf{I} + \frac{\mathbf{\tilde{x}}
\mathbf{\tilde{x}}^\top}{\sigma^2}\Big|^{1/d}\Big]=-\mathbb{E}\ln\Big|\mathbf{I}
+ \frac{\mathbf{\tilde{x}} \mathbf{\tilde{x}}^\top}{\sigma^2}\Big| ,
\nonumber
\end{eqnarray} and thereby
\begin{equation}
-\ln \Big|\mathbf{I} + \frac{\mathbb{E}(\mathbf{\tilde{x}}
\mathbf{\tilde{x}}^\top)}{\sigma^2}\Big| \leq
-\mathbb{E}\ln\Big|\mathbf{I} + \frac{\mathbf{\tilde{x}}
\mathbf{\tilde{x}}^\top}{\sigma^2}\Big|.\nonumber
\end{equation}
\end{proof}

Denote $R=\sup_\mathbf{\tilde{x}}\|\mathbf{\tilde{x}}\|$. From
inequality (\ref{eqn11022}), we can finally prove the following
theorem, as detailed in Appendix~\ref{appB1}.

\begin{theorem}[Multi-view PAC-Bayes bound 1]
     Consider a classifier prior given in (\ref{prior6}) and a classifier posterior given in
     (\ref{prior7}).
     For any data distribution $\mathcal{D}$, for
     any $\delta\in(0,1]$, with probability at least
$1-\delta$ over $S\sim \mathcal{D}^m$, the following inequality
holds
\begin{eqnarray}
&&\forall \mathbf{w}, \mu: KL_+(\hat{E}_{Q,S}||E_{Q,\mathcal{D}})\leq\nonumber\\
&&\frac{-\frac{d}{2} \ln\Big[ f_m - (\sqrt[d]{(R/\sigma)^2 +1} - 1)
\sqrt{\frac{1}{2m}\ln\frac{3}{\delta}}~\Big]_+
+\frac{H_m}{2\sigma^2}
 + \frac{(1+\mu^2) R^2}{2\sigma^2}
\sqrt{\frac{1}{2m}\ln\frac{3}{\delta}} + \frac{\mu^2}{2}
+\ln\big(\frac{m+1}{\delta/3}\big)}{m},\nonumber
\end{eqnarray}
where
\begin{eqnarray}
f_m&=&\frac{1}{m}\sum_{i=1}^m \Big|\mathbf{I} +
\frac{\mathbf{\tilde{x}}_i
\mathbf{\tilde{x}}_i^\top}{\sigma^2}\Big|^{1/d}, \nonumber\\
H_m&=& \frac{1}{m}\sum_{i=1}^m [\mathbf{\tilde{x}}_i^\top
\mathbf{\tilde{x}}_i + \mu^2
(\mathbf{w}^\top\mathbf{\tilde{x}}_i)^2] , \nonumber
\end{eqnarray}
and $\|\mathbf{w}\|=1$.
\label{thmMvPB1}
\end{theorem}

From the bound formulation, we see that if $ (\mathbf{w}^\top\mathbf{\tilde{x}}_i)^2 $ is small, that is,
if the two view outputs tend to agree, the bound will be tight.

Note that, although the formulation of $f_m$ involves the outer
product of feature vectors, it can actually be represented by the
inner product, which is obvious through the following determinant
equality
\begin{equation}
\label{eqop2ip}
\Big|\mathbf{I} + \frac{\mathbf{\tilde{x}}_i
\mathbf{\tilde{x}}_i^\top}{\sigma^2}\Big| =
\frac{\mathbf{\tilde{x}}_i^\top \mathbf{\tilde{x}}_i}{\sigma^2} + 1
, \nonumber
\end{equation}
where we have used the fact that matrix $\mathbf{\tilde{x}}_i
\mathbf{\tilde{x}}_i^\top$ has rank $1$ and has only one nonzero
eigenvalue.

We can use inequality (\ref{eqn11023}) instead of (\ref{eqn11022})
to derive a $d$-independent bound (see Theorem~\ref{thmMvPB2}
below), which is independent of the dimensionality of the feature
representation space.

\begin{theorem}[Multi-view PAC-Bayes bound 2]
     Consider a classifier prior given in (\ref{prior6}) and a classifier posterior given in
     (\ref{prior7}).
     For any data distribution $\mathcal{D}$, for
     any $\delta\in(0,1]$, with probability at least
$1-\delta$ over $S\sim \mathcal{D}^m$, the following inequality
holds
\begin{eqnarray}
\forall \mathbf{w}, \mu: KL_+(\hat{E}_{Q,S}||E_{Q,\mathcal{D}})\leq
\frac{\tilde{f}/{2}
 + \frac{1}{2} \Big(\frac{(1+\mu^2)R^2}{\sigma^2}+ \ln(1+ \frac{R^2}{\sigma^2}
)\Big)\sqrt{\frac{1}{2m}\ln\frac{2}{\delta}} + \frac{\mu^2}{2}
+\ln\big(\frac{m+1}{\delta/2}\big)}{m}, \nonumber
\end{eqnarray}
where
\begin{equation}
\tilde{f}= \frac{1}{m}\sum_{i=1}^m \Big( \frac{1}{\sigma^2}
[\mathbf{\tilde{x}}_i^\top \mathbf{\tilde{x}}_i + \mu^2
(\mathbf{w}^\top\mathbf{\tilde{x}}_i)^2]-\ln\Big|\mathbf{I} +
\frac{\mathbf{\tilde{x}}_i \mathbf{\tilde{x}}_i^\top}{\sigma^2}\Big|
\Big) , \nonumber
\end{equation}
and $\|\mathbf{w}\|=1$. \label{thmMvPB2}
\end{theorem}
The proof of this theorem is given in Appendix~\ref{appB2}.

Since this bound is independent with $ d $ and the term $ \Big|\mathbf{I} + \frac{\mathbf{\tilde{x}}_i
\mathbf{\tilde{x}}_i^\top}{\sigma^2}\Big| $ involving the outer product can be represented by the inner product
through (\ref{eqop2ip}), this bound can be employed when the dimension of the kernelized feature space goes to infinity.

\section{Another Two Multi-view PAC-Bayes Bounds}
\label{Last2bounds}

We further propose a new prior whose center is not located at the
origin, inspired by \citet{EmilioJMLR12}. The new classifier prior
is
\begin{equation}
P(\mathbf{u}) \propto \mathcal{N}(\eta\mathbf{w}_p,
\mathbf{I})\times V(\mathbf{u}_1,\mathbf{u}_2), \label{prior34}
\end{equation}
and the posterior is still
\begin{equation}
Q(\mathbf{u}) = \mathcal{N}(\mu \mathbf{w}, \mathbf{I}),
\label{prior35}
\end{equation}
where $\eta>0$, $\|\mathbf{w}\|=1$ and
$\mathbf{w}_p=\mathbb{E}_{(\mathbf{x},y)\sim\mathcal{D}}[y \mathbf{x}]$ (or
$\mathbb{E}_{(\mathbf{x},y)\sim\mathcal{D}}[y \phi(\mathbf{x})]$ in a predefined
kernel space) with $\mathbf{x}=[\mathbf{x}_1^\top,
\mathbf{x}_2^\top]^\top$.

We have
\begin{eqnarray}
P(\mathbf{u}) &\propto& \mathcal{N}(\eta\mathbf{w}_p,
\mathbf{I})\times
V(\mathbf{u}_1,\mathbf{u}_2) \nonumber\\
&\propto& \exp\left\{-\frac{1}{2} (\mathbf{u}-\eta\mathbf{w}_p)^\top
(\mathbf{u}-\eta\mathbf{w}_p)\right\} \times \exp\left\{-
\frac{1}{2\sigma^2}\mathbf{u}^\top \mathbb{E}(\mathbf{\tilde{x}}
\mathbf{\tilde{x}}^\top) \mathbf{u} \right\}.\nonumber
\end{eqnarray}
That is, $P(\mathbf{u})=\mathcal{N}(\mathbf{u}_p, \Sigma)$ with
$\Sigma=\left(\mathbf{I} + \frac{\mathbb{E}(\mathbf{\tilde{x}}
\mathbf{\tilde{x}}^\top)}{\sigma^2}\right)^{-1}$ and $\mathbf{u}_p=
\eta \Sigma \mathbf{w}_p$.

With $d$ being the dimensionality of $\mathbf{u}$, the KL divergence
between the posterior and prior is
\begin{eqnarray}
&&KL(Q(\mathbf{u})\|P(\mathbf{u}))\nonumber\\
&=&\frac{1}{2}\left(-\ln (\Big|\mathbf{I} +
\frac{\mathbb{E}(\mathbf{\tilde{x}}
\mathbf{\tilde{x}}^\top)}{\sigma^2}\Big|) + \mbox{tr}(\mathbf{I} +
\frac{\mathbb{E}(\mathbf{\tilde{x}}
\mathbf{\tilde{x}}^\top)}{\sigma^2}) + (\mathbf{u}_p-\mu
\mathbf{w})^\top (\mathbf{I} + \frac{\mathbb{E}(\mathbf{\tilde{x}}
\mathbf{\tilde{x}}^\top)}{\sigma^2}) (\mathbf{u}_p-\mu \mathbf{w}) - d \right) \nonumber\\
&=& \frac{1}{2}\left(-\ln (\Big|\mathbf{I} +
\frac{\mathbb{E}(\mathbf{\tilde{x}}
\mathbf{\tilde{x}}^\top)}{\sigma^2}\Big|) +\frac{1}{\sigma^2}
\mathbb{E}[\mathbf{\tilde{x}}^\top \mathbf{\tilde{x}}] +
(\mathbf{u}_p-\mu \mathbf{w})^\top (\mathbf{I} +
\frac{\mathbb{E}(\mathbf{\tilde{x}}
\mathbf{\tilde{x}}^\top)}{\sigma^2}) (\mathbf{u}_p-\mu
\mathbf{w})\right). \label{eqn1}
\end{eqnarray}
We have
\begin{eqnarray}
&& (\mathbf{u}_p-\mu \mathbf{w})^\top (\mathbf{I} +
\frac{\mathbb{E}(\mathbf{\tilde{x}}
\mathbf{\tilde{x}}^\top)}{\sigma^2}) (\mathbf{u}_p-\mu
\mathbf{w})\nonumber\\
&=& \eta^2 \mathbf{w}_p^\top (\mathbf{I} +
\frac{\mathbb{E}(\mathbf{\tilde{x}}
\mathbf{\tilde{x}}^\top)}{\sigma^2})^{-1} \mathbf{w}_p -2\eta\mu
\mathbf{w}_p^\top\mathbf{w} +  \mu^2 \mathbf{w}^\top (\mathbf{I} +
\frac{\mathbb{E}(\mathbf{\tilde{x}}
\mathbf{\tilde{x}}^\top)}{\sigma^2}) \mathbf{w} \nonumber\\
&=& \eta^2 \mathbf{w}_p^\top (\mathbf{I} +
\frac{\mathbb{E}(\mathbf{\tilde{x}}
\mathbf{\tilde{x}}^\top)}{\sigma^2})^{-1} \mathbf{w}_p -2\eta\mu
\mathbf{w}_p^\top\mathbf{w} + \frac{\mu^2}{\sigma^2}
\mathbb{E}[(\mathbf{w}^\top\mathbf{\tilde{x}})^2] +
\mu^2\nonumber\\
&=& \eta^2 \mathbf{w}_p^\top (\mathbf{I} +
\frac{\mathbb{E}(\mathbf{\tilde{x}}
\mathbf{\tilde{x}}^\top)}{\sigma^2})^{-1} \mathbf{w}_p -2\eta\mu
\mathbb{E}[y(\mathbf{w}^\top\mathbf{x})] + \frac{\mu^2}{\sigma^2}
\mathbb{E}[(\mathbf{w}^\top\mathbf{\tilde{x}})^2] +
\mu^2 \nonumber\\
& \leq & \eta^2 \mathbf{w}_p^\top \mathbf{w}_p -2\eta\mu
\mathbb{E}[y(\mathbf{w}^\top\mathbf{x})] + \frac{\mu^2}{\sigma^2}
\mathbb{E}[(\mathbf{w}^\top\mathbf{\tilde{x}})^2] + \mu^2 ,
\label{eqn2}
\end{eqnarray}
where for the last inequality  we have used the fact that matrix
$\mathbf{I}- (\mathbf{I} + \frac{\mathbb{E}(\mathbf{\tilde{x}}
\mathbf{\tilde{x}}^\top)}{\sigma^2})^{-1}$ is symmetric and positive
semi-definite.

Define $\mathbf{\hat{w}}_p = \mathbb{E}_{(\mathbf{x},y)\sim
S}[y\mathbf{x}]=\frac{1}{m}\sum_{i=1}^m[y_i\mathbf{x}_i]$. We have
\begin{eqnarray}
\eta^2\mathbf{w}_p^\top
\mathbf{w}_p&=&\|\eta\mathbf{w}_p-\mu\mathbf{w}+
\mu\mathbf{w}\|^2 \nonumber\\
&=& \|\eta\mathbf{w}_p-\mu\mathbf{w} \|^2 +
\mu^2 + 2 (\eta\mathbf{w}_p-\mu\mathbf{w})^\top \mu\mathbf{w} \nonumber\\
&\leq& \|\eta\mathbf{w}_p-\mu\mathbf{w} \|^2 +
\mu^2 + 2\mu \|\eta\mathbf{w}_p-\mu\mathbf{w}\|  \nonumber\\
&=&(\|\eta\mathbf{w}_p-\mu\mathbf{w} \| + \mu)^2. \label{eqn3}
\end{eqnarray}
Moreover, we have
\begin{equation}
\|\eta\mathbf{w}_p-\mu\mathbf{w} \|=\|\eta\mathbf{w}_p-
\eta\mathbf{\hat{w}}_p+\eta\mathbf{\hat{w}}_p-\mu\mathbf{w} \|\leq
\|\eta\mathbf{w}_p-
\eta\mathbf{\hat{w}}_p\|+\|\eta\mathbf{\hat{w}}_p-\mu\mathbf{w} \|.
\label{eqn6}
\end{equation}

From (\ref{eqn1}), (\ref{eqn2}), (\ref{eqn3}) and (\ref{eqn6}), it
follows that
\begin{eqnarray}
KL(Q(\mathbf{u})\|P(\mathbf{u})) &\leq& -\frac{1}{2}\ln
(\Big|\mathbf{I} + \frac{\mathbb{E}(\mathbf{\tilde{x}}
\mathbf{\tilde{x}}^\top)}{\sigma^2}\Big|)
+\frac{1}{2}(\|\eta\mathbf{w}_p-
\eta\mathbf{\hat{w}}_p\|+\|\eta\mathbf{\hat{w}}_p-\mu\mathbf{w} \| + \mu)^2 + \nonumber\\
&& \frac{1}{2\sigma^2}\mathbb{E}\left[ \mathbf{\tilde{x}}^\top
\mathbf{\tilde{x}}  -2\eta\mu \sigma^2 y(\mathbf{w}^\top\mathbf{x})
+ \mu^2 (\mathbf{w}^\top\mathbf{\tilde{x}})^2 \right] +\frac{\mu^2
}{2}. \label{eqn7}
\end{eqnarray}

By using inequalities (\ref{eqn11022}) and (\ref{eqn11023}), we get
the following two theorems, whose proofs are detailed in
Appendix~\ref{appB3} and Appendix~\ref{appB4}, respectively.

\begin{theorem}[Multi-view PAC-Bayes bound 3]
     Consider a classifier prior given in (\ref{prior34}) and a classifier posterior given in
     (\ref{prior35}).
     For any data distribution $\mathcal{D}$, for any $\mathbf{w}$, positive $\mu$, and positive $\eta$, for
     any $\delta\in(0,1]$, with probability at least
$1-\delta$ over $S\sim \mathcal{D}^m$ the following multi-view
PAC-Bayes bound holds
\begin{eqnarray}
& & KL_+(\hat{E}_{Q,S}||E_{Q,\mathcal{D}}) \leq \frac{-\frac{d}{2}\ln\Big[
f_m - (\sqrt[d]{(R/\sigma)^2 +1} - 1)
\sqrt{\frac{1}{2m}\ln\frac{4}{\delta}}~\Big]_+}{m} +\nonumber\\
&& \frac{\frac{1}{2}\left(\frac{\eta
R}{\sqrt{m}}\Big(2+\sqrt{2\ln\frac{4}{\delta}}\Big)+
\|\eta\mathbf{\hat{w}}_p-\mu\mathbf{w} \| + \mu\right)^2
+\frac{\hat{H}_m}{2\sigma^2}
 + \frac{R^2+ \mu^2
R^2+4\eta\mu\sigma^2 R}{2\sigma^2}
\sqrt{\frac{1}{2m}\ln\frac{4}{\delta}} + \frac{\mu^2}{2}
+\ln\big(\frac{m+1}{\delta/4}\big)}{m}, \nonumber
\end{eqnarray}
where
\begin{eqnarray}
f_m&=&\frac{1}{m}\sum_{i=1}^m \Big|\mathbf{I} +
\frac{\mathbf{\tilde{x}}_i
\mathbf{\tilde{x}}_i^\top}{\sigma^2}\Big|^{1/d}, \nonumber\\
\hat{H}_m&=&\frac{1}{m}\sum_{i=1}^m [\mathbf{\tilde{x}}_i^\top
\mathbf{\tilde{x}}_i -2\eta\mu \sigma^2
y_i(\mathbf{w}^\top\mathbf{x}_i)+ \mu^2
(\mathbf{w}^\top\mathbf{\tilde{x}}_i)^2],\nonumber
\end{eqnarray}
and $\|\mathbf{w}\|=1$. \label{thmMvPB3}
\end{theorem}

Besides the term $ (\mathbf{w}^\top\mathbf{\tilde{x}}_i)^2 $ that appears in the previous bounds,
we can see that if $ \|\eta\mathbf{\hat{w}}_p-\mu\mathbf{w} \| $ is small, that is, the centers of the prior and posterior tend to
overlap, the bound will be tight.

\begin{theorem}[Multi-view PAC-Bayes bound 4]
     Consider a classifier prior given in (\ref{prior34}) and a classifier posterior given in
     (\ref{prior35}).
     For any data distribution $\mathcal{D}$, for any $\mathbf{w}$, positive $\mu$, and positive $\eta$, for
     any $\delta\in(0,1]$, with probability at least
$1-\delta$ over $S\sim \mathcal{D}^m$ the following multi-view
PAC-Bayes bound holds
\begin{eqnarray}
 KL_+(\hat{E}_{Q,S}||E_{Q,\mathcal{D}})&\leq&
\frac{\frac{1}{2}\left(\frac{\eta
R}{\sqrt{m}}\Big(2+\sqrt{2\ln\frac{3}{\delta}}\Big)+
\|\eta\mathbf{\hat{w}}_p-\mu\mathbf{w} \| + \mu\right)^2}{m} +\nonumber\\
&& \frac{\frac{\tilde{H}_m}{2}
 + \frac{R^2+4\eta\mu\sigma^2 R+
\mu^2 R^2 +\sigma^2\ln(1+\frac{R^2}{\sigma^2})}{2\sigma^2}
\sqrt{\frac{1}{2m}\ln\frac{3}{\delta}} + \frac{\mu^2}{2}
+\ln\big(\frac{m+1}{\delta/3}\big)}{m}, \nonumber
\end{eqnarray}
where
\begin{equation}
\tilde{H}_m=\frac{1}{m}\sum_{i=1}^m [\frac{\mathbf{\tilde{x}}_i^\top
\mathbf{\tilde{x}}_i -2\eta\mu \sigma^2
y_i(\mathbf{w}^\top\mathbf{x}_i)+ \mu^2
(\mathbf{w}^\top\mathbf{\tilde{x}}_i)^2}{\sigma^2} - \ln
\Big|\mathbf{I} + \frac{\mathbf{\tilde{x}}_i
\mathbf{\tilde{x}}_i^\top}{\sigma^2}\Big|], \nonumber
\end{equation}
and $\|\mathbf{w}\|=1$. \label{thmMvPB4}
\end{theorem}

\section{Separate Training Data Dependent Multi-view PAC-Bayes Bounds}
\label{stdbounds}

We attempt to improve our bounds by using a separate set of training data to determine new priors, inspired by
\citet{Ambroladze07} and \citet{EmilioJMLR12}. We consider a spherical Gaussian whose center is calculated on a subset $ T $ of training set
comprising $ r $ training patterns and labels.
In the experiments this is taken as a random subset, but for simplicity of the presentation we will assume $ T $ comprises
the last $ r $ examples $ \{\mathbf{x}_k,y_k\}^m_{k=m-r+1} $.

The new prior is
\begin{equation}
P(\mathbf{u}) = \mathcal{N}(\eta\mathbf{w}_p,
\mathbf{I}), \label{prior56}
\end{equation}
and the posterior is again
\begin{equation}
Q(\mathbf{u}) = \mathcal{N}(\mu \mathbf{w}, \mathbf{I}).
\label{posterior56}
\end{equation}
One reasonable choice of $ \mathbf{w}_p $ is
\begin{equation}
\mathbf{w}_p = \left(\mathbb{E}_{\mathbf{\tilde{x}}}[\mathbf{\tilde{x}}
\mathbf{\tilde{x}}^\top]\right)^{-1}
\mathbb{E}_{(\mathbf{x},y)\sim\mathcal{D}}[y \mathbf{x}],
\end{equation}
which is the solution to the following optimization problem
\begin{equation}
\max_{\mathbf{w}}
\frac{\mathbb{E}_{\mathbf{x}_1,y}[y\mathbf{w}_1^\top\mathbf{x}_1]+
\mathbb{E}_{\mathbf{x}_2,y}[y\mathbf{w}_2^\top\mathbf{x}_2]}
{\mathbb{E}_{\mathbf{x}_1,\mathbf{x}_2}[(\mathbf{w}_1^\top\mathbf{x}_1-\mathbf{w}_2^\top\mathbf{x}_2)^2]},
\end{equation}
where $ \mathbf{w} = [\mathbf{w}_1^\top, \mathbf{w}_2^\top]^\top $.
We use the subset $ T $ to approximate $ \mathbf{w}_p $, that is,
let
\begin{eqnarray}
\mathbf{w}_p &=& \left(\mathbb{E}_{\mathbf{\tilde{x}} \sim T}[\mathbf{\tilde{x}}
\mathbf{\tilde{x}}^\top]\right)^{-1}
\mathbb{E}_{(\mathbf{x},y) \sim T}[y \mathbf{x}] \nonumber\\
&=& \left(\frac{1}{m-r}\sum_{k=r}^{m-r+1}[\mathbf{\tilde{x}}_k
\mathbf{\tilde{x}}_k^\top]\right)^{-1}
\frac{1}{m-r}\sum_{k=r}^{m-r+1}[y_k \mathbf{x}_k]. \label{prior56wp}
\end{eqnarray}

The KL divergence between the posterior and prior is
\begin{equation}
KL(Q(\mathbf{u})\|P(\mathbf{u})) = KL(\mathcal{N}(\mu\mathbf{w},I)\|\mathcal{N}(\eta\mathbf{w}_p,I))
=\|\eta\mathbf{w}_p-\mu\mathbf{w} \|^2.
\end{equation}

Since we separate $ r $ examples to calculate the prior, the actual size of training set that we apply bound to
is $ m-r $. We have the following bound.

\begin{theorem}[Multi-view PAC-Bayes bound 5]
     Consider a classifier prior given in (\ref{prior56}) and a classifier posterior given in
     (\ref{posterior56}), with $\mathbf{w}_p$ given in (\ref{prior56wp}).
     For any data distribution $\mathcal{D}$, for any $\mathbf{w}$, positive $\mu$, and positive $\eta$, for
     any $\delta\in(0,1]$, with probability at least
$1-\delta$ over $S\sim \mathcal{D}^m$ the following multi-view
PAC-Bayes bound holds
\begin{eqnarray}
& & KL_+(\hat{E}_{Q,S}||E_{Q,\mathcal{D}}) \leq
\frac{\frac{1}{2}\|\eta\mathbf{w}_p-\mu\mathbf{w} \|^2+\ln\frac{m-r+1}{\delta}}{m-r}
\end{eqnarray}
and $\|\mathbf{w}\|=1$. \label{thmMvPB5}
\end{theorem}

Another choice of $ \mathbf{w}_p $ is to learn a multi-view SVM classifier
with the subset $ T $, leading to the following bound.

\begin{theorem}[Multi-view PAC-Bayes bound 6]
     Consider a classifier prior given in (\ref{prior56}) and a classifier posterior given in
     (\ref{posterior56}). Classifier $ \mathbf{w}_p $ has been learned from a subset T of
     r examples a priori separated from a training set S of m samples.
     For any data distribution $\mathcal{D}$, for any $\mathbf{w}$, positive $\mu$, and positive $\eta$, for
     any $\delta\in(0,1]$, with probability at least
$1-\delta$ over $S\sim \mathcal{D}^m$ the following multi-view
PAC-Bayes bound holds
\begin{eqnarray}
& & KL_+(\hat{E}_{Q,S}||E_{Q,\mathcal{D}}) \leq
\frac{\frac{1}{2}\|\eta\mathbf{w}_p-\mu\mathbf{w} \|^2+\ln\frac{m-r+1}{\delta}}{m-r}
\end{eqnarray}
and $\|\mathbf{w}\|=1$. \label{thmMvPB5}
\end{theorem}

Although the above two bounds look similar, they are essentially
different in that the priors are determined differently. We will see
in the experimental results that they also perform differently when
applied in our experiments.

\section{Semi-supervised Multi-view PAC-Bayes Bounds}
\label{secSSLBs}

Now we consider PAC-Bayes analysis for semi-supervised multi-view
learning, where besides the $m$ labeled examples we are further
provided with $u$ unlabeled examples
$U=\{\tilde{\mathbf{x}}_j\}_{j=m+1}^{m+u}$. We replace
$V(\mathbf{u}_1,\mathbf{u}_2)$ with $\hat{V}(\mathbf{u}_1,\mathbf{u}_2)$, which has the form
\begin{equation}
\label{eqnVssl} \hat{V}(\mathbf{u}_1,\mathbf{u}_2) = \exp\left\{-
\frac{1}{2\sigma^2}\mathbf{u}^\top {\mathbb{E}}_U(\mathbf{\tilde{x}}
\mathbf{\tilde{x}}^\top) \mathbf{u} \right\} ,
\end{equation} where ${\mathbb{E}}_U$ means the empirical
average over the unlabeled set $U$.

\subsection{Noninformative Prior Center}
Under a similar setting with Section~\ref{First2bounds}, that is,
$P(\mathbf{u}) \propto \mathcal{N}(\mathbf{0}, \mathbf{I})\times
\hat{V}(\mathbf{u}_1,\mathbf{u}_2)$, we have
$P(\mathbf{u})=\mathcal{N}(\mathbf{0}, \Sigma)$ with
$\Sigma=\left(\mathbf{I} + \frac{{\mathbb{E}}_U(\mathbf{\tilde{x}}
\mathbf{\tilde{x}}^\top)}{\sigma^2}\right)^{-1}$. Therefore,
according to Theorem~\ref{thmKLpq}, we have
\begin{equation}
KL(Q(\mathbf{u})\|P(\mathbf{u}))= \frac{1}{2}\left(-\ln
(\Big|\mathbf{I} + \frac{{\mathbb{E}}_U(\mathbf{\tilde{x}}
\mathbf{\tilde{x}}^\top)}{\sigma^2}\Big|) + \frac{1}{\sigma^2}
{\mathbb{E}}_U [\mathbf{\tilde{x}}^\top \mathbf{\tilde{x}} + \mu^2
(\mathbf{w}^\top\mathbf{\tilde{x}})^2]+ \mu^2\right)
.\label{eqnklpq2}
\end{equation}

Substituting (\ref{eqnklpq2}) into Theorem~\ref{LemLangford}, we reach
the following semi-supervised multi-view PAC-Bayes bound.

\begin{theorem}[Semi-supervised multi-view PAC-Bayes bound 1]
     Consider a classifier prior given in (\ref{prior6}) with $ \hat{V} $ defined in (\ref{eqnVssl}), a classifier posterior given in
     (\ref{prior7}) and an unlabeled set $U=\{\tilde{\mathbf{x}}_j\}_{j=m+1}^{m+u}$.
     For any data distribution $\mathcal{D}$, for
     any $\delta\in(0,1]$, with probability at least
$1-\delta$ over $S\sim \mathcal{D}^m$, the following inequality
holds
\begin{eqnarray}
\hspace{-40mm}&&\forall \mathbf{w}, \mu: KL_+(\hat{E}_{Q,S}||E_{Q,\mathcal{D}})\leq\nonumber\\
\hspace{-40mm}&&\frac{ \frac{1}{2}\left(-\ln (\Big|\mathbf{I} +
\frac{{\mathbb{E}}_U(\mathbf{\tilde{x}}
\mathbf{\tilde{x}}^\top)}{\sigma^2}\Big|) + \frac{1}{\sigma^2}
{\mathbb{E}}_U [\mathbf{\tilde{x}}^\top \mathbf{\tilde{x}} + \mu^2
(\mathbf{w}^\top\mathbf{\tilde{x}})^2]+ \mu^2\right)
+\ln\big(\frac{m+1}{\delta}\big)}{m},\nonumber
\end{eqnarray}
where $\|\mathbf{w}\|=1$. \label{thmSMvPB1}
\end{theorem}

\subsection{Informative Prior Center}

Similar to Section~\ref{Last2bounds}, we take the classifier prior
to be
\begin{equation}
P(\mathbf{u}) \propto \mathcal{N}(\eta\mathbf{w}_p,
\mathbf{I})\times \hat{V}(\mathbf{u}_1,\mathbf{u}_2), \label{eqn141129-1}
\end{equation}
where $\hat{V}(\mathbf{u}_1,\mathbf{u}_2)$ is
given by (\ref{eqnVssl}), $\eta>0$ and $\mathbf{w}_p=\mathbb{E}_{(\mathbf{x},y)\sim\mathcal{D}}[y
\mathbf{x}]$ with $\mathbf{x}=[\mathbf{x}_1^\top,
\mathbf{x}_2^\top]^\top$. We have
$P(\mathbf{u})=\mathcal{N}(\mathbf{u}_p, \Sigma)$ with
$\Sigma=\left(\mathbf{I} + \frac{{\mathbb{E}}_U(\mathbf{\tilde{x}}
\mathbf{\tilde{x}}^\top)}{\sigma^2}\right)^{-1}$ and $\mathbf{u}_p=
\eta \Sigma \mathbf{w}_p$.

By similar reasoning, we get
\begin{eqnarray}
KL(Q(\mathbf{u})\|P(\mathbf{u})) &\leq& -\frac{1}{2}\ln
(\Big|\mathbf{I} + \frac{{\mathbb{E}}_U(\mathbf{\tilde{x}}
\mathbf{\tilde{x}}^\top)}{\sigma^2}\Big|)
+\frac{1}{2}(\|\eta\mathbf{w}_p-
\eta\mathbf{\hat{w}}_p\|+\|\eta\mathbf{\hat{w}}_p-\mu\mathbf{w} \| + \mu)^2 + \nonumber\\
&& \frac{1}{2\sigma^2}{\mathbb{E}}_U\left[ \mathbf{\tilde{x}}^\top
\mathbf{\tilde{x}} + \mu^2 (\mathbf{w}^\top\mathbf{\tilde{x}})^2
\right] -\eta\mu \mathbb{E}\left[
y(\mathbf{w}^\top\mathbf{x})\right]+\frac{\mu^2}{2},
\end{eqnarray}
which is analogous to (\ref{eqn7}).

Then, we can give the following semi-supervised multi-view PAC-Bayes
bound, whose proof is provided in Appendix~\ref{appSB2}.
\begin{theorem}[Semi-supervised multi-view PAC-Bayes bound 2]
     Consider a classifier prior given in (\ref{eqn141129-1}) with $ \hat{V} $ defined in (\ref{eqnVssl}), a classifier posterior given in
     (\ref{prior35}) and an unlabeled set $U=\{\tilde{\mathbf{x}}_j\}_{j=m+1}^{m+u}$.
     For any data distribution $\mathcal{D}$,  for any $\mathbf{w}$, positive $\mu$, and positive $\eta$, for
     any $\delta\in(0,1]$, with probability at least
$1-\delta$ over $S\sim \mathcal{D}^m$, the following inequality
holds
\begin{eqnarray}
 &&KL_+(\hat{E}_{Q,S}||E_{Q,\mathcal{D}})\leq
\frac{\frac{1}{2}\left(\frac{\eta
R}{\sqrt{m}}\Big(2+\sqrt{2\ln\frac{3}{\delta}}\Big)+
\|\eta\mathbf{\hat{w}}_p-\mu\mathbf{w} \| + \mu\right)^2}{m} +\nonumber\\
&& \frac{\frac{1}{2}\left(-\ln (\Big|\mathbf{I} +
\frac{{\mathbb{E}}_U(\mathbf{\tilde{x}}
\mathbf{\tilde{x}}^\top)}{\sigma^2}\Big|) + \frac{1}{\sigma^2}
{\mathbb{E}}_U [\mathbf{\tilde{x}}^\top \mathbf{\tilde{x}} + \mu^2
(\mathbf{w}^\top\mathbf{\tilde{x}})^2]+ \mu^2\right) +\bar{S}_m +
\eta\mu R \sqrt{\frac{2}{m}\ln\frac{3}{\delta}}
+\ln\big(\frac{m+1}{\delta/3}\big)}{m}, \nonumber
\end{eqnarray}
where
\begin{eqnarray}
\bar{S}_m=\frac{1}{m}\sum_{i=1}^m [-\eta\mu
y_i(\mathbf{w}^\top\mathbf{x}_i)], \nonumber
\end{eqnarray}
and $\|\mathbf{w}\|=1$. \label{thmSMvPB2}
\end{theorem}

\section{Learning Algorithms}
\label{Algorithms}

Below we provide the optimization formulations for the single-view
and multi-view SVMs as well as semi-supervised multi-view SVMs that
are adopted to train classifiers and calculate PAC-Bayes bounds.
Note that the augmented vector representation is used by appending a
scalar $1$ at the end of the feature representations, in order to
formulate the classifier in a simple form without the explicit bias
term.

\subsection{SVMs}
The optimization problem \citep{cristianini00book,taylorS11} is formulated as
\begin{eqnarray}
 \min_{\mathbf{w}, \bm\xi} &&
\frac{1}{2}\|\mathbf{w}\|^2 +
                              C \sum_{i=1}^n\xi_i \nonumber\\
\mbox{s.t.} && y_i(\mathbf{w}^\top \mathbf{x}_i )\geq 1-\xi_i,
\quad i=1,\ldots, n , \nonumber\\
&&\xi_i\geq 0, \quad i=1,\ldots, n , \label{eqn10mar2606}
\end{eqnarray}
where scalar $C$ controls the balance between the margin and
empirical loss. This problem is a differentiable convex problem with
affine constraints. The constraint qualification is satisfied by the
refined Slater's condition.

The Lagrangian of problem~(\ref{eqn10mar2606}) is
\begin{eqnarray}
L(\mathbf{w},  \bm\xi, \bm\lambda,
\bm\gamma)&=&\frac{1}{2}\|\mathbf{w}\|^2 +C \sum_{i=1}^n
\xi_i-\sum_{i=1}^n \lambda_i \left[y_i(\mathbf{w}^\top \mathbf{x}_i) -1 + \xi_i\right] \nonumber\\
&&- \sum_{i=1}^n \gamma_i\xi_i, \quad \lambda_i\geq 0, \quad
\gamma_i \geq 0,
\end{eqnarray}
where $\bm\lambda=[\lambda_1,\ldots,\lambda_n]^\top$ and
$\bm\gamma=[\gamma_1,\ldots,\gamma_n]^\top$  are the associated
Lagrange multipliers. From the optimality conditions, we obtain
\begin{eqnarray}
&&\partial_{\mathbf{w}}L(\mathbf{w}^*, b^*, \bm\xi^*, \bm\lambda^*,
\bm\gamma^*)=
\mathbf{w}^*-\sum_{i=1}^n \lambda_i^* y_i \mathbf{x}_i=0, \label{eqn110601}\\
\label{eqn10mar2607} && \partial_{\xi_i}L(\mathbf{w}^*, b^*,
\bm\xi^*, \bm\lambda^*, \bm\gamma^*)=C-\lambda_i^*-\gamma_i^*=0,
\quad i=1,\ldots, n.
\end{eqnarray}

The dual optimization problem is derived as
\begin{eqnarray}
\label{eqnMar2801} \min_{\bm\lambda} && \frac{1}{2}
\bm\lambda^\top D \bm\lambda - \bm\lambda^\top \mathbf{1} \nonumber\\
\mbox{s.t.}
&& \bm\lambda\succeq 0, \nonumber\\
&& \bm\lambda\preceq C\mathbf{1} ,
\end{eqnarray}
where $D$ is a symmetric $n\times n$ matrix with entries $D_{ij}=y_i
y_j\mathbf{x}_i^\top \mathbf{x}_j$. Once the solution $\bm\lambda^*$
is given, the SVM decision function is given by
\begin{equation}
c^*(\textbf{x})= \mbox{sign}\left(\sum_{i=1}^n y_i \lambda^*_i
\mathbf{x}^\top\mathbf{x}_i \right). \nonumber
\end{equation}
Using the kernel trick, the optimization problem for SVMs is still
(\ref{eqnMar2801}). However, now $D_{ij}=y_i y_j
\kappa(\mathbf{x}_i, \mathbf{x}_j)$ with the kernel function
$\kappa(\cdot,\cdot)$, and the solution for the SVM classifier is
formulated as
\begin{equation}
c^*(\textbf{x})= \mbox{sign}\left(\sum_{i=1}^n y_i \lambda^*_i
\kappa(\mathbf{x}_i,\mathbf{x}) \right). \nonumber
\end{equation}

\subsection{MvSVMs}

Denote the classifier weights from two views by $\mathbf{w}_1$ and
$\mathbf{w}_2$ which are not assumed to be unit vectors at the
moment. Inspired by semi-supervised multi-view SVMs \citep{sindhwani05,sindhwani08,Sun10JMLR}, the objective function of the multi-view SVMs (MvSVMs)
can be given by
\begin{eqnarray}
\min_{\mathbf{w}_1, \mathbf{w}_2, \bm\xi_1, \bm\xi_2}
&&\frac{1}{2}(\|\mathbf{w}_1\|^2 +\|\mathbf{w}_2\|^2)  +  C_1
\sum_{i=1}^n (\xi_{1}^{i} +\xi_{2}^{i}) + C_2 \sum_{i=1}^n
(\mathbf{w}_1^\top\mathbf{{x}}_{1}^{i}
-\mathbf{w}_2^\top\mathbf{{x}}_{2}^{i})^2 \nonumber\\
\mbox{s.t.} && y_i\mathbf{w}_1^\top\mathbf{x}_{1}^i\geq 1
-\xi_{1}^i, \quad i=1,
\cdots, n, \nonumber\\
&& y_i\mathbf{w}_2^\top\mathbf{x}_{2}^i\geq 1-\xi_{2}^i, \quad i=1,
\cdots, n, \nonumber\\
&& \xi_{1}^i, ~\xi_{2}^i \geq 0, \quad i=1, \cdots, n.
\label{eqnOptPro}
\end{eqnarray}

If kernel functions are used, the solution of the above optimization
problem can be given by $\mathbf{w}_1=\sum_{i=1}^n \alpha_1^i
k_1(\mathbf{x}_{1}^{i},\cdot)$, and $\mathbf{w}_2=\sum_{i=1}^n
\alpha_2^i k_2(\mathbf{x}_{2}^{i},\cdot)$. Since a function defined on view $j$ only depends on
the $j$th feature set, the solution is given by
\begin{equation}
\mathbf{w}_1=\sum_{i=1}^n \alpha_1^i k_1(\mathbf{x}_{i},\cdot),
\quad \mathbf{w}_2=\sum_{i=1}^n \alpha_2^i k_2(\mathbf{x}_{i},
\cdot) . \label{eqn1310111}
\end{equation}

It can be shown that
\begin{eqnarray}
&&\|\mathbf{w}_1\|^2 = \bm\alpha_1^\top K_1\bm\alpha_1, \quad \|\mathbf{w}_2\|^2 = \bm\alpha_2^\top K_2\bm\alpha_2 ,\nonumber\\
&&\sum_{i=1}^n (\mathbf{w}_1^\top \mathbf{x}_i -\mathbf{w}_2^\top
\mathbf{x}_i)^2 = (K_1\bm\alpha_1-K_2\bm\alpha_2)^\top
(K_1\bm\alpha_1-K_2\bm\alpha_2), \nonumber
\end{eqnarray}
where $K_1$ and $K_2$ are kernel matrices from two views.

The optimization problem (\ref{eqnOptPro}) can be reformulated as
the following
\begin{eqnarray}
\min_{\bm\alpha_1, \bm\alpha_2, \bm\xi_1, \bm\xi_2}
&&F_0=\frac{1}{2}(\bm\alpha_1^\top K_1\bm\alpha_1 + \bm\alpha_2^\top
K_2\bm\alpha_2) + C_2 (K_1\bm\alpha_1-K_2\bm\alpha_2)^\top
(K_1\bm\alpha_1-K_2\bm\alpha_2) +\nonumber\\
&& C_1\sum_{i=1}^n (\xi_{1}^i +\xi_{2}^i)\nonumber\\
\mbox{s.t.} && y_i\Big(\sum_{j=1}^n \alpha_1^j k_1(\mathbf{x}_j,
\mathbf{x}_i)\Big)\geq 1 -\xi_{1}^{i}, \quad i=1,
\cdots, n, \nonumber\\
&& y_i \Big(\sum_{j=1}^n \alpha_2^j k_2(\mathbf{x}_j,
\mathbf{x}_i)\Big) \geq 1-\xi_{2}^{i}, \quad i=1,
\cdots, n, \nonumber\\
&& \xi_{1}^{i}, ~\xi_{2}^{i} \geq 0, \quad i=1, \cdots, n.
\label{eqnOptProblem}
\end{eqnarray}

The derivation of the dual optimization formulation is detailed in
Appendix~\ref{appMvSVM}.  Table~\ref{TableMvLapSVM} summarizes the
MvSVM algorithm.

\begin{table}[thb]
\caption{The MvSVM Algorithm} %\vskip 0.03in
\label{TableMvLapSVM} \centering
\begin{tabular}{p{370pt}}%{l}
\hline \noalign{\smallskip}
{\bf Input:}\\
\quad A training set with $n$ examples $\{(\mathbf{x}_i,y_i)\}_{i=1}^n$ (each example has two views).\\
\quad Kernel function $k_1(\cdot,\cdot)$ and $k_2(\cdot,\cdot)$ for two views, respectively.\\
\quad Regularization coefficients  $C_1, C_2$.\\
%\noalign{\smallskip}
{\bf Algorithm:}\\
$\;\; 1 \;$ Calculate Gram matrices $K_1$ and $K_2$ from two views.\\
$\;\; 2 \;$ Calculate $A, B, D$ according to~(\ref{eqn07031}).\\
$\;\; 3 \;$ Solve the quadratic optimization
problem~(\ref{eqn07032})
to get $\boldsymbol{\lambda}_{1}$, $\boldsymbol{\lambda}_{2}$.\\
$\;\; 4 \;$ Calculate $\boldsymbol{\alpha}_1$ and
$\boldsymbol{\alpha}_2$ using (\ref{eqnalpha1}) and
(\ref{eqnalpha2}).\\
%\noalign{\smallskip}
{\bf Output:} Classifier parameters $\boldsymbol{\alpha}_{1}$ and
$\boldsymbol{\alpha}_{2}$ used
by (\ref{eqn1310111}).\\
\noalign{\smallskip} \hline
\end{tabular}
%\vskip -0.2in
\end{table}

\subsection{Semi-supervised MvSVMs (SMvSVMs)}
Next we give the optimization formulation for semi-supervised MvSVMs
(SMvSVMs) \citep{sindhwani05,sindhwani08,Sun10JMLR}, where besides the $n$ labeled examples we further have
$u$ unlabeled examples.

Denote the classifier weights from two views by $\mathbf{w}_1$ and
$\mathbf{w}_2$ which are not assumed to be unit vectors. The
objective function of SMvSVMs is
\begin{eqnarray}
\min_{\mathbf{w}_1, \mathbf{w}_2, \bm\xi_1, \bm\xi_2}
&&\frac{1}{2}(\|\mathbf{w}_1\|^2 +\|\mathbf{w}_2\|^2)  +  C_1
\sum_{i=1}^n (\xi_{1}^{i} +\xi_{2}^{i}) + C_2 \sum_{i=1}^{n+u}
(\mathbf{w}_1^\top\mathbf{{x}}_{1}^{i}
-\mathbf{w}_2^\top\mathbf{{x}}_{2}^{i})^2 \nonumber\\
\mbox{s.t.} && y_i\mathbf{w}_1^\top\mathbf{x}_{1}^i\geq 1
-\xi_{1}^i, \quad i=1,
\cdots, n, \nonumber\\
&& y_i\mathbf{w}_2^\top\mathbf{x}_{2}^i\geq 1-\xi_{2}^i, \quad i=1,
\cdots, n, \nonumber\\
&& \xi_{1}^i, ~\xi_{2}^i \geq 0, \quad i=1, \cdots, n.
\label{eqnOptProS}
\end{eqnarray}

If kernel functions are used, the solution can be expressed by
$\mathbf{w}_1=\sum_{i=1}^{n+u} \alpha_1^i
k_1(\mathbf{x}_{1}^{i},\cdot)$, and $\mathbf{w}_2=\sum_{i=1}^{n+u}
\alpha_2^i k_2(\mathbf{x}_{2}^{i},\cdot)$. Since a function defined on view $j$ only depends on
the $j$th feature set,
the solution is given by
\begin{equation}
\mathbf{w}_1=\sum_{i=1}^{n+u} \alpha_1^i k_1(\mathbf{x}_{i},\cdot),
\quad \mathbf{w}_2=\sum_{i=1}^{n+u} \alpha_2^i k_2(\mathbf{x}_{i},
\cdot) . \label{eqn1310111S}
\end{equation}

It is straightforward to show that
\begin{eqnarray}
&&\|\mathbf{w}_1\|^2 = \bm\alpha_1^\top K_1\bm\alpha_1, \quad \|\mathbf{w}_2\|^2 = \bm\alpha_2^\top K_2\bm\alpha_2 ,\nonumber\\
&&\sum_{i=1}^{n+u} (\mathbf{w}_1^\top \mathbf{x}_i
-\mathbf{w}_2^\top \mathbf{x}_i)^2 =
(K_1\bm\alpha_1-K_2\bm\alpha_2)^\top
(K_1\bm\alpha_1-K_2\bm\alpha_2), \nonumber
\end{eqnarray}
where $(n+u)\times (n+u)$ matrices $K_1$ and $K_2$ are kernel
matrices from two views.

The optimization problem (\ref{eqnOptProS}) can be reformulated as
\begin{eqnarray}
\min_{\bm\alpha_1, \bm\alpha_2, \bm\xi_1, \bm\xi_2}
&&\tilde{F}_0=\frac{1}{2}(\bm\alpha_1^\top K_1\bm\alpha_1 +
\bm\alpha_2^\top K_2\bm\alpha_2) + C_2
(K_1\bm\alpha_1-K_2\bm\alpha_2)^\top
(K_1\bm\alpha_1-K_2\bm\alpha_2) +\nonumber\\
&& C_1\sum_{i=1}^n (\xi_{1}^i +\xi_{2}^i)\nonumber\\
\mbox{s.t.} && y_i\Big(\sum_{j=1}^{n+u} \alpha_1^j k_1(\mathbf{x}_j,
\mathbf{x}_i)\Big)\geq 1 -\xi_{1}^{i}, \quad i=1,
\cdots, n, \nonumber\\
&& y_i \Big(\sum_{j=1}^{n+u} \alpha_2^j k_2(\mathbf{x}_j,
\mathbf{x}_i)\Big) \geq 1-\xi_{2}^{i}, \quad i=1,
\cdots, n, \nonumber\\
&& \xi_{1}^{i}, ~\xi_{2}^{i} \geq 0, \quad i=1, \cdots, n.
\label{eqnOptProblemS}
\end{eqnarray}

The derivation of the dual optimization formulation is detailed in
Appendix~\ref{appSMvSVM}.  Table~\ref{TableSMvLapSVM} summarizes the
SMvSVM algorithm.

\begin{table}[thb]
\caption{The SMvSVM Algorithm} %\vskip 0.03in
\label{TableSMvLapSVM} \centering
\begin{tabular}{p{370pt}}%{l}
\hline \noalign{\smallskip}
{\bf Input:}\\
\quad A training set with $n$ examples $\{(\mathbf{x}_i,y_i)\}_{i=1}^n$ (each example has two views) and $u$ unlabeled
     examples.\\
\quad Kernel function $k_1(\cdot,\cdot)$ and $k_2(\cdot,\cdot)$ for two views, respectively.\\
\quad Regularization coefficients  $C_1, C_2$.\\
%\noalign{\smallskip}
{\bf Algorithm:}\\
$\;\; 1 \;$ Calculate Gram matrices $K_1$ and $K_2$ from two views.\\
$\;\; 2 \;$ Calculate $A, B, D$ according to~(\ref{eqn07031S}).\\
$\;\; 3 \;$ Solve the quadratic optimization
problem~(\ref{eqn07032S})
to get $\boldsymbol{\lambda}_{1}$, $\boldsymbol{\lambda}_{2}$.\\
$\;\; 4 \;$ Calculate $\boldsymbol{\alpha}_1$ and
$\boldsymbol{\alpha}_2$ using (\ref{eqnalpha1S}) and
(\ref{eqnalpha2S}).\\
%\noalign{\smallskip}
{\bf Output:} Classifier parameters $\boldsymbol{\alpha}_{1}$ and
$\boldsymbol{\alpha}_{2}$ used
by (\ref{eqn1310111S}).\\
\noalign{\smallskip} \hline
\end{tabular}
%\vskip -0.2in
\end{table}

\section{Experiments}
\label{expriments}

The new bounds are evaluated on one synthetic and three real-world
multi-view data sets where the learning task is binary
classification. Below we first introduce the used data and the
experimental settings. Then we report the test errors of the
involved variants of the SVM algorithms, and evaluate the usefulness
and relative performance of the new PAC-Bayes bounds.

\subsection{Data Sets}

The four multi-view data sets are introduced as follows.

\subsubsection*{Synthetic}
The synthetic data include $2000$ examples half of which belong to
the positive class. The dimensionality for each of the two views is
$50$. We first generate two random direction vectors one for each
view, and then for each view sample $2000$ points to make the inner
products between the direction and  the feature vector of half of
the points be positive and the inner products for the other half of
the points be negative. For the same point, the corresponding inner
products calculated from the two views are made identical. Finally,
we add Gaussian white noise to the generated data to form the
synthetic data set.

\subsubsection*{Handwritten}
The handwritten digit data set is taken from the UCI machine
learning repository~\citep{UCIdata}, which  includes features of ten
handwritten digits ($0 \sim 9$) extracted from a collection of Dutch
utility maps. It consists of 2000 examples (200 examples per class)
with the first view  being the 76 Fourier coefficients, and the
second view being the 64 Karhunen-Lo\`{e}ve coefficients of each
image. Binary classification between digits (1, 2, 3) and  (4, 5, 6)
is used for experiments.

\subsubsection*{Ads}
The ads data are used for classifying web images into ads and
non-ads~\citep{Kushmerick99}. This data set consists of 3279
examples with 459 of them being ads. 1554 binary attributes (weights
of text terms related to an image using Boolean model) are used for
classification, whose values can be 0 and 1. These attributes are
divided into two views: one view describes the image itself (terms
in the image's caption, URL and alt text) and the other view
contains features from other information (terms in the page and
destination URLs). The two views have 587 and 967 features,
respectively.

\subsubsection*{Course}
The course data set consists of 1051 two-view web pages collected
from computer science department web sites at four universities:
Cornell University, University of Washington, University of
Wisconsin, and University of Texas. There are 230 course pages and
821 non-course pages. The two views are words occurring in a web
page and words appearing in the links pointing to that
page~\citep{Blum98,Sun10JMLR}. The document vectors are normalized
to $tf$-$idf$ (term frequency-inverse document frequency) features and then principal component analysis is used
to perform dimensionality reduction. The dimensions of the two views
are 500 and 87, respectively.

\subsection{Experimental Settings}

Our experiments include algorithm test error evaluation and
PAC-Bayes bound evaluation for single-view learning, multi-view
learning, supervised learning and semi-supervised learning. For
single-view learning, SVMs are trained separately on each of the two
 views and the third view (concatenating the previous two views to form a long view),
 providing three supervised classifiers which are
called SVM-1, SVM-2 and SVM-3, respectively. Evaluating the
performance of the third view is interesting to compare single-view
and multi-view learning methods, since single-view learning on the
third view can exploit the same data as the usual multi-view
learning algorithms. The MvSVMs and SMvSVMs are supervised
multi-view learning and semi-supervised multi-view learning
algorithms, respectively. The linear kernel is used for all the
algorithms.

For each data set, four experimental settings are used. All the settings
use $ 20\% $ of all the examples as the unlabeled examples. For the remaining examples,
the four settings use $ 20\% $, $ 40\% $, $ 60\% $ and $ 80\% $ of them as the labeled training set,
respectively, and the rest forms the test set. Supervised algorithms will not use the unlabeled training data.
For multi-view PAC-Bayes bound 5 and 6, we use $ 20\% $ of the labeled training set
to calculate the prior, and evaluate the bounds on the remaining $ 80\% $ of training set.
Each setting involves 10 random partitions of the above subsets.
The reported performance is the average test error and standard deviation
over these random partitions.

Model parameters, i.e., $C$ in SVMs, and $C_1, C_2$ in MvSVMs and
SMvSVMs, are selected by three-fold cross-validation on each labeled
training set, where $C_1, C_2$ are selected from $\{10^{-6}, 10^{-4}, 10^{-2}, 1, 10, 100\}$
and $C$ is selected from $\{10^{-8}, 5\times10^{-8}, 10^{-7}, 5\times10^{-7}, 10^{-6}, 5\times10^{-6}, 10^{-5}, 5\times10^{-5}, 10^{-4}, 5\times10^{-4}, 10^{-3}, 5\times10^{-3}, 10^{-2}, 5\times10^{-2}, 10^{-1}, 5\times10^{-1}, 1, 5, 10, 20, 25, 30, 40, 50, 55, 60, 70, 80, 85, 90, 100, 300, 500, 700, 900, 1000\}$.
All the PAC-Bayes bounds are evaluated with a confidence of
$\delta=0.05$. We normalize $ \mathbf{w} $ in the posterior
when we calculate the bounds. For multi-view PAC-Bayes bounds, $\sigma$ is fixed to
100, $\eta$ is set to $1$, and $R$ is equal to $1$ which is clear
from the augmented feature representation and  data normalization
preprocessing (all the training examples after feature augmentation
are divided by a common value to make the maximum feature vector
length be one).

We evaluate the following eleven PAC-Bayes bounds where the last
eight bounds are presented in this paper.
\begin{itemize}
\item PB-1: The PAC-Bayes bound given by Theorem~\ref{thmStaPB} and the SVM
algorithm on the first view.
\item PB-2: The PAC-Bayes bound given by Theorem~\ref{thmStaPB} and the SVM
algorithm on the second view.
\item PB-3: The PAC-Bayes bound given by Theorem~\ref{thmStaPB} and the SVM
algorithm on the third view.
\item MvPB-1: Multi-view PAC-Bayes bound 1 with the MvSVM algorithm.
\item MvPB-2: Multi-view PAC-Bayes bound 2 with the MvSVM algorithm.
\item MvPB-3: Multi-view PAC-Bayes bound 3 with the MvSVM algorithm.
\item MvPB-4: Multi-view PAC-Bayes bound 4 with the MvSVM algorithm.
\item MvPB-5: Multi-view PAC-Bayes bound 5 with the MvSVM algorithm.
\item MvPB-6: Multi-view PAC-Bayes bound 6 with the MvSVM algorithm.
\item SMvPB-1: Semi-supervised multi-view PAC-Bayes bound 1 with the
SMvSVM algorithm.
\item SMvPB-2: Semi-supervised multi-view PAC-Bayes bound 2 with the
SMvSVM algorithm.
\end{itemize}

\subsection{Test Errors}
The prediction performances of SVMs, MvSVMs and SMvSVMs for the four
experimental settings are reported in Table~\ref{tbacc1}, Table~\ref{tbacc2}, Table~\ref{tbacc3} and
Table~\ref{tbacc4}, respectively. For each data set, the best
performance is indicated with boldface numbers. From all of these
results, we see that MvSVMS and SMvSVMs have the best overall performance and
sometimes single-view SVMs can have the best performances.
SMvSVMs often perform better than MvSVMS since additional unlabeled examples are used,
especially when the labeled training data set is small.
Moreover, as expected, with more labeled training data
the prediction performance of the algorithms will
usually increase.

\begin{table*}[!b]
 \small
 \centering
 \begin{tabular}{|c|c|c|c|c|}
 \hline
 Test Error &
 Synthetic &
 Handwritten &
 Ads &
 Course \\
 \hline
 SVM-1 &
 $ 17.20 \pm 1.39 $ &
 $ 5.66 \pm 0.94 $ &
 $ 5.84 \pm 0.56 $ &
 $ 19.15 \pm 1.54 $ \\
\hline
 SVM-2 &
 $ 19.98 \pm 0.76 $ &
 $ 3.98 \pm 0.68 $ &
 $ 5.25 \pm 0.79 $ &
 $ \bm{10.15} \pm 1.60 $ \\
\hline
 SVM-3 &
 $ 16.55 \pm 2.04 $ &
 $ \bm{1.65} \pm 0.53 $ &
 $ 4.62 \pm 0.80 $ &
 $ 10.33 \pm 1.34 $ \\
\hline
 MvSVM &
 $ 10.54 \pm 0.73 $ &
 $ 2.17 \pm 0.64 $ &
 $ \bm{4.55} \pm 0.66 $ &
 $ 10.55 \pm 1.47 $ \\
\hline
 SMvSVM &
 $ \bm{10.30} \pm 0.79 $ &
 $ 2.04 \pm 0.69 $ &
 $ 4.70 \pm 0.70 $ &
 $ 10.28 \pm 1.63 $ \\
\hline
 \end{tabular}
 \caption{Average error rates (\%) and standard deviations for different learning algorithms under the 20\% training setting. }
 \label{tbacc1}
\end{table*}

\begin{table*}[!hp]
 \small
 \centering
 \begin{tabular}{|c|c|c|c|c|}
 \hline
 Test Error &
 Synthetic &
 Handwritten &
 Ads &
 Course \\
 \hline
 SVM-1 &
 $ 14.49 \pm 0.98 $ &
 $ 5.57 \pm 0.41 $ &
 $ 5.04 \pm 0.83 $ &
 $ 14.23 \pm 1.27 $ \\
\hline
 SVM-2 &
 $ 16.88 \pm 1.06 $ &
 $ 3.75 \pm 0.99 $ &
 $ 4.14 \pm 0.40 $ &
 $ 7.64 \pm 0.80 $ \\
\hline
 SVM-3 &
 $ 10.31 \pm 0.82 $ &
 $ \bm{1.51} \pm 0.39 $ &
 $ 3.61 \pm 0.54 $ &
 $ 7.68 \pm 0.97 $ \\
\hline
 MvSVM &
 $ 7.72 \pm 0.78 $ &
 $ 1.98 \pm 0.61 $ &
 $ 3.56 \pm 0.54 $ &
 $ 7.00 \pm 0.93 $ \\
\hline
 SMvSVM &
 $ \bm{7.48} \pm 0.66 $ &
 $ 2.03 \pm 0.61 $ &
 $ \bm{3.44} \pm 0.54 $ &
 $ \bm{6.81} \pm 0.98 $ \\
\hline
 \end{tabular}
 \caption{Average error rates (\%) and standard deviations for different learning algorithms under the 40\% training setting. }
 \label{tbacc2}
\end{table*}

\begin{table*}[!hp]
 \small
 \centering
 \begin{tabular}{|c|c|c|c|c|}
 \hline
 Test Error &
 Synthetic &
 Handwritten &
 Ads &
 Course \\
 \hline
 SVM-1 &
 $ 14.23 \pm 1.24 $ &
 $ 5.16 \pm 0.61 $ &
 $ 4.32 \pm 0.50 $ &
 $ 11.28 \pm 1.30 $ \\
\hline
 SVM-2 &
 $ 16.11 \pm 0.94 $ &
 $ 3.46 \pm 0.94 $ &
 $ 3.90 \pm 0.58 $ &
 $ 6.53 \pm 1.44 $ \\
\hline
 SVM-3 &
 $ 9.08 \pm 1.07 $ &
 $ 1.77 \pm 0.85 $ &
 $ 3.43 \pm 0.51 $ &
 $ 6.62 \pm 1.33 $ \\
\hline
 MvSVM &
 $ \bm{7.30} \pm 0.85 $ &
 $ \bm{1.67} \pm 0.63 $ &
 $ 3.45 \pm 0.32 $ &
 $ \bm{5.82} \pm 1.73 $ \\
\hline
 SMvSVM &
 $ 7.31 \pm 0.80 $ &
 $ 1.82 \pm 0.70 $ &
 $ \bm{3.36} \pm 0.38 $ &
 $ 5.93 \pm 1.63 $ \\
\hline
 \end{tabular}
 \caption{Average error rates (\%) and standard deviations for different learning algorithms under the 60\% training setting. }
 \label{tbacc3}
\end{table*}

\begin{table*}[!hp]
 \small
 \centering
 \begin{tabular}{|c|c|c|c|c|}
 \hline
 Test Error &
 Synthetic &
 Handwritten &
 Ads &
 Course \\
 \hline
 SVM-1 &
 $ 13.06 \pm 2.00 $ &
 $ 5.42 \pm 1.51 $ &
 $ 4.47 \pm 0.60 $ &
 $ 9.70 \pm 1.64 $ \\
\hline
 SVM-2 &
 $ 16.03 \pm 1.73 $ &
 $ 3.54 \pm 1.33 $ &
 $ 3.59 \pm 0.66 $ &
 $ 5.62 \pm 1.68 $ \\
\hline
 SVM-3 &
 $ 8.06 \pm 1.11 $ &
 $ 1.93 \pm 0.66 $ &
 $ \bm{2.96} \pm 0.51 $ &
 $ 5.56 \pm 1.72 $ \\
\hline
 MvSVM &
 $ \bm{6.28} \pm 1.20 $ &
 $ \bm{1.82} \pm 0.75 $ &
 $ 3.19 \pm 0.63 $ &
 $ 4.20 \pm 1.51 $ \\
\hline
 SMvSVM &
 $ \bm{6.28} \pm 1.19 $ &
 $ 1.93 \pm 0.77 $ &
 $ 3.15 \pm 0.75 $ &
 $ \bm{3.96} \pm 1.59 $ \\
\hline
 \end{tabular}
 \caption{Average error rates (\%) and standard deviations for different learning algorithms under the 80\% training setting. }
 \label{tbacc4}
\end{table*}

\begin{table*}
 \small
 \centering
 \begin{tabular}{|c|c|c|c|c|}
 \hline
 PAC-Bayes Bound &
 Synthetic &
 Handwritten &
 Ads &
 Course \\
 \hline
 PB-1 &
 $ 60.58 \pm 0.12 $ &
 $ 54.61 \pm 1.59 $ &
 $ 40.49 \pm 2.09 $ &
 $ \bm{58.93} \pm 8.90 $ \\
\hline
 PB-2 &
 $ 60.72 \pm 0.09 $ &
 $ \bm{45.17} \pm 3.74 $ &
 $ \bm{40.44} \pm 2.12 $ &
 $ 61.64 \pm 1.49 $ \\
\hline
 PB-3 &
 $ \bm{60.49} \pm 0.12 $ &
 $ 47.62 \pm 3.42 $ &
 $ 43.75 \pm 3.15 $ &
 $ 59.67 \pm 2.32 $ \\
\hline
\hline
 MvPB-1 &
 $ 61.27 \pm 0.07 $ &
 $ 51.63 \pm 2.89 $ &
 $ 40.87 \pm 2.77 $ &
 $ 63.54 \pm 0.45 $ \\
\hline
 MvPB-2 &
 $ 61.04 \pm 0.07 $ &
 $ 51.45 \pm 2.89 $ &
 $ 40.80 \pm 2.77 $ &
 $ 63.26 \pm 0.47 $ \\
\hline
 MvPB-3 &
 $ 62.35 \pm 0.01 $ &
 $ 63.44 \pm 0.62 $ &
 $ 56.38 \pm 1.49 $ &
 $ 66.37 \pm 0.06 $ \\
\hline
 MvPB-4 &
 $ 62.17 \pm 0.01 $ &
 $ 63.23 \pm 0.61 $ &
 $ 56.29 \pm 1.48 $ &
 $ 66.14 \pm 0.06 $ \\
\hline
 MvPB-5 &
 $ 61.84 \pm 0.09 $ &
 $ 52.52 \pm 3.01 $ &
 $ 43.21 \pm 2.94 $ &
 $ 64.36 \pm 0.43 $ \\
\hline
 MvPB-6 &
 $ 63.74 \pm 0.08 $ &
 $ 58.65 \pm 7.09 $ &
 $ 54.94 \pm 4.68 $ &
 $ 67.75 \pm 0.25 $ \\
\hline
 SMvPB-1 &
 $ \underline{60.60} \pm 0.06 $ &
 $ \underline{49.84} \pm 2.87 $ &
 $ \underline{40.65} \pm 3.25 $ &
 $ \underline{62.77} \pm 0.49 $ \\
\hline
 SMvPB-2 &
 $ 62.17 \pm 0.01 $ &
 $ 62.94 \pm 0.62 $ &
 $ 56.28 \pm 1.30 $ &
 $ 66.14 \pm 0.06 $ \\
\hline
 \end{tabular}
 \caption{Average PAC-Bayes bounds (\%) and standard deviations for different learning algorithms under the 20\% training setting. }
 \label{tbpb1}
\end{table*}

\begin{table*}
 \small
 \centering
 \begin{tabular}{|c|c|c|c|c|}
 \hline
 PAC-Bayes Bound &
 Synthetic &
 Handwritten &
 Ads &
 Course \\
 \hline
 PB-1 &
 $ 57.20 \pm 0.05 $ &
 $ 45.26 \pm 1.48 $ &
 $ 33.11 \pm 3.89 $ &
 $ 59.68 \pm 0.52 $ \\
\hline
 PB-2 &
 $ 57.40 \pm 0.11 $ &
 $ \bm{35.45} \pm 3.22 $ &
 $ \bm{28.85} \pm 3.26 $ &
 $ \bm{55.26} \pm 1.97 $ \\
\hline
 PB-3 &
 $ 57.15 \pm 0.07 $ &
 $ 35.48 \pm 2.26 $ &
 $ 32.74 \pm 4.29 $ &
 $ 56.12 \pm 0.78 $ \\
\hline
\hline
 MvPB-1 &
 $ 57.69 \pm 0.09 $ &
 $ 40.85 \pm 3.23 $ &
 $ 33.36 \pm 2.17 $ &
 $ 59.17 \pm 0.51 $ \\
\hline
 MvPB-2 &
 $ 57.54 \pm 0.08 $ &
 $ \underline{40.76} \pm 3.22 $ &
 $ \underline{33.32} \pm 2.17 $ &
 $ 58.99 \pm 0.50 $ \\
\hline
 MvPB-3 &
 $ 58.97 \pm 0.02 $ &
 $ 57.26 \pm 1.17 $ &
 $ 51.68 \pm 1.38 $ &
 $ 61.91 \pm 0.07 $ \\
\hline
 MvPB-4 &
 $ 58.85 \pm 0.02 $ &
 $ 57.15 \pm 1.16 $ &
 $ 51.62 \pm 1.37 $ &
 $ 61.77 \pm 0.10 $ \\
\hline
 MvPB-5 &
 $ 57.44 \pm 0.13 $ &
 $ 42.56 \pm 3.36 $ &
 $ 35.86 \pm 2.23 $ &
 $ 59.91 \pm 0.48 $ \\
\hline
 MvPB-6 &
 $ \underline{\bm{52.67}} \pm 2.36 $ &
 $ 42.57 \pm 5.93 $ &
 $ 47.34 \pm 3.05 $ &
 $ 62.86 \pm 0.09 $ \\
\hline
 SMvPB-1 &
 $ 57.27 \pm 0.06 $ &
 $ \underline{40.76} \pm 3.26 $ &
 $ 34.26 \pm 3.00 $ &
 $ \underline{58.69} \pm 0.44 $ \\
\hline
 SMvPB-2 &
 $ 58.85 \pm 0.01 $ &
 $ 57.22 \pm 1.18 $ &
 $ 52.16 \pm 1.50 $ &
 $ 61.77 \pm 0.09 $ \\
\hline
 \end{tabular}
 \caption{Average PAC-Bayes bounds (\%) and standard deviations for different learning algorithms under the 40\% training setting. }
 \label{tbpb2}
\end{table*}

\begin{table*}
 \small
 \centering
 \begin{tabular}{|c|c|c|c|c|}
 \hline
 PAC-Bayes Bound &
 Synthetic &
 Handwritten &
 Ads &
 Course \\
 \hline
 PB-1 &
 $ 55.45 \pm 0.08 $ &
 $ 42.07 \pm 2.35 $ &
 $ 29.65 \pm 1.93 $ &
 $ 57.52 \pm 0.22 $ \\
\hline
 PB-2 &
 $ 55.71 \pm 0.08 $ &
 $ 30.70 \pm 2.05 $ &
 $ \bm{28.59} \pm 3.71 $ &
 $ \bm{53.71} \pm 2.27 $ \\
\hline
 PB-3 &
 $ 55.39 \pm 0.16 $ &
 $ \bm{30.50} \pm 3.31 $ &
 $ 30.49 \pm 4.35 $ &
 $ 53.78 \pm 1.01 $ \\
\hline
\hline
 MvPB-1 &
 $ 55.89 \pm 0.08 $ &
 $ 34.16 \pm 1.88 $ &
 $ 31.72 \pm 4.13 $ &
 $ 56.90 \pm 0.46 $ \\
\hline
 MvPB-2 &
 $ 55.78 \pm 0.07 $ &
 $ 34.09 \pm 1.88 $ &
 $ \underline{31.69} \pm 4.13 $ &
 $ 56.75 \pm 0.45 $ \\
\hline
 MvPB-3 &
 $ 57.38 \pm 0.01 $ &
 $ 52.82 \pm 1.08 $ &
 $ 49.77 \pm 2.49 $ &
 $ 59.82 \pm 0.07 $ \\
\hline
 MvPB-4 &
 $ 57.29 \pm 0.01 $ &
 $ 52.73 \pm 1.07 $ &
 $ 49.74 \pm 2.48 $ &
 $ 59.69 \pm 0.07 $ \\
\hline
 MvPB-5 &
 $ 55.60 \pm 0.08 $ &
 $ 36.17 \pm 1.88 $ &
 $ 34.11 \pm 4.26 $ &
 $ 57.56 \pm 0.42 $ \\
\hline
 MvPB-6 &
 $ \underline{\bm{39.20}} \pm 5.03 $ &
 $ \underline{31.76} \pm 4.17 $ &
 $ 47.56 \pm 3.81 $ &
 $ 60.67 \pm 0.05 $ \\
\hline
 SMvPB-1 &
 $ 55.58 \pm 0.06 $ &
 $ 33.93 \pm 2.00 $ &
 $ 32.33 \pm 3.37 $ &
 $ \underline{56.53} \pm 0.43 $ \\
\hline
 SMvPB-2 &
 $ 57.28 \pm 0.01 $ &
 $ 52.76 \pm 1.15 $ &
 $ 50.51 \pm 1.64 $ &
 $ 59.69 \pm 0.07 $ \\
\hline
 \end{tabular}
 \caption{Average PAC-Bayes bounds (\%) and standard deviations for different learning algorithms under the 60\% training setting. }
 \label{tbpb3}
\end{table*}

\begin{table*}
 \small
 \centering
 \begin{tabular}{|c|c|c|c|c|}
 \hline
 PAC-Bayes Bound &
 Synthetic &
 Handwritten &
 Ads &
 Course \\
\hline
 PB-1 &
 $ 54.64 \pm 0.77 $ &
 $ 37.52 \pm 1.42 $ &
 $ \bm{28.97} \pm 1.51 $ &
 $ 56.21 \pm 0.18 $ \\
\hline
 PB-2 &
 $ 54.59 \pm 0.04 $ &
 $ 28.47 \pm 2.07 $ &
 $ 30.28 \pm 1.83 $ &
 $ \bm{51.28} \pm 2.97 $ \\
\hline
 PB-3 &
 $ 54.21 \pm 0.08 $ &
 $ \bm{26.50} \pm 2.15 $ &
 $ 29.74 \pm 3.42 $ &
 $ 52.00 \pm 0.85 $ \\
\hline
\hline
 MvPB-1 &
 $ 54.65 \pm 0.05 $ &
 $ 30.25 \pm 0.86 $ &
 $ 29.69 \pm 0.84 $ &
 $ 55.77 \pm 1.09 $ \\
\hline
 MvPB-2 &
 $ 54.63 \pm 0.05 $ &
 $ 30.19 \pm 0.86 $ &
 $ \underline{29.67} \pm 0.84 $ &
 $ 55.38 \pm 0.50 $ \\
\hline
 MvPB-3 &
 $ 56.41 \pm 0.00 $ &
 $ 49.51 \pm 0.52 $ &
 $ 48.12 \pm 0.94 $ &
 $ 58.55 \pm 0.07 $ \\
\hline
 MvPB-4 &
 $ 56.32 \pm 0.01 $ &
 $ 49.43 \pm 0.54 $ &
 $ 48.09 \pm 0.92 $ &
 $ 58.44 \pm 0.07 $ \\
\hline
 MvPB-5 &
 $ 54.36 \pm 0.05 $ &
 $ 32.39 \pm 0.88 $ &
 $ 31.44 \pm 0.98 $ &
 $ 56.22 \pm 0.41 $ \\
\hline
 MvPB-6 &
 $ \underline{\bm{26.89}} \pm 2.05 $ &
 $ 31.52 \pm 3.33 $ &
 $ 46.31 \pm 1.50 $ &
 $ 59.23 \pm 0.18 $ \\
\hline
 SMvPB-1 &
 $ 54.41 \pm 0.03 $ &
 $ \underline{30.15} \pm 0.79 $ &
 $ 30.55 \pm 2.28 $ &
 $ \underline{55.24} \pm 0.43 $ \\
\hline
 SMvPB-2 &
 $ 56.32 \pm 0.01 $ &
 $ 49.43 \pm 0.46 $ &
 $ 48.77 \pm 1.38 $ &
 $ 58.44 \pm 0.06 $ \\
\hline
 \end{tabular}
 \caption{Average PAC-Bayes bounds (\%) and standard deviations for different learning algorithms under the 80\% training setting. }
 \label{tbpb4}
\end{table*}

\subsection{PAC-Bayes Bounds}

Table~\ref{tbpb1}, Table~\ref{tbpb2}, Table~\ref{tbpb3} and Table~\ref{tbpb4} show the values of
various PAC-Bayes bounds under different settings, where for each data set the best bound is
indicated in bold and the best multi-view bound is indicated with underline.

From all the bound results, we find that the best single-view bound
is usually tighter than the best multi-view bound, expect on the synthetic data set.
One possible explanation for this is that, the synthetic data set is ideal and
in accordance with the assumptions for multi-view learning encoded in the prior, while the
real world data sets are not.
This also indicates that there is much
space and possibility for further developments of multi-view
PAC-Bayes analysis. In addition, with more labeled training data the
corresponding bound will usually become tighter. Last but not least,
among the eight presented multi-view PAC-Bayes bounds on real world data sets, the tightest
one is often the first semi-supervised multi-view bound which exploits
unlabeled data to calculate the function $\hat{V}(\mathbf{u}_1,
\mathbf{u}_2)$ and needs no further relaxation. The results
also show that the second multi-view PAC-Bayes bound
(dimensionality-independent bound with the prior distribution
centered at the origin) is sometimes very good.

\section{Conclusion}
\label{conclusion}

The paper lays the foundation of a theoretical and practical
framework for defining priors that encode non-trivial interactions
between data distributions and classifiers and translating them into
sophisticated regularization schemes and associated generalization
bounds. Specifically, we have presented eight new multi-view PAC-Bayes
bounds, which integrate the view agreement as a key measure to
modulate the prior distributions of classifiers. As extensions of
PAC-Bayes analysis to the multi-view learning scenario, the proposed
theoretical results are promising to fill the gap between the
developments in theory and practice of multi-view learning, and are
also possible to serve as the underpinnings to explain the
effectiveness of multi-view learning. We have validated the theoretical
superiority of multi-view learning in the ideal case of synthetic data,
though this is not so evident for real world data which may not well meet our
assumptions on the priors for multi-view learning.

The usefulness of the proposed bounds has been shown. Although
often the current bounds are not the tightest, they indeed open
the possibility of applying PAC-Bayes analysis to multi-view
learning. We think the set of bounds could be further tightened in
the future by adopting other techniques. It is also possible to study
algorithms whose co-regularization term pushes towards the minimization of the multi-view PAC-Bayes bounds.
In addition, we may use the work in this paper to motivate PAC-Bayes
analysis for other learning tasks such as multi-task learning and
domain adaptation, since these tasks are closely related to the
current multi-view learning.

%-------------------

% Acknowledgements should go at the end, before appendices and references

\acks{This work is supported by the National Natural Science
Foundation of China under Project 61370175, the Scientific Research
Foundation for the Returned Overseas Chinese Scholars, State
Education Ministry, and Shanghai Knowledge Service Platform Project
(No. ZF1213).}

\appendix
\section{Proof of Theorem~\ref{thmMvPB1}}
\label{appB1}

Define
\begin{equation}
f(\mathbf{\tilde{x}}_1,\ldots,\mathbf{\tilde{x}}_m)=\frac{1}{m}\sum_{i=1}^m
\Big|\mathbf{I} + \frac{\mathbf{\tilde{x}}_i
\mathbf{\tilde{x}}_i^\top}{\sigma^2}\Big|^{1/d} . \nonumber
\end{equation}
Since the rank of matrix $\mathbf{\tilde{x}}_i
\mathbf{\tilde{x}}_i^\top/\sigma^2$ is 1 with the nonzero eigenvalue
being $\|\mathbf{\tilde{x}}_i\|^2/\sigma^2$ and the determinant of a
positive semi-definite matrix is equal to the product of its
eigenvalues, it follows that
\begin{eqnarray}
&&\sup_{\mathbf{\tilde{x}}_1,\ldots,\mathbf{\tilde{x}}_m,
\mathbf{\bar{x}}_i}|f(\mathbf{\tilde{x}}_1,\ldots,\mathbf{\tilde{x}}_m)-f(\mathbf{\tilde{x}}_1,\ldots,
\mathbf{\bar{x}}_i, \mathbf{\tilde{x}}_{i+1},\ldots,
\mathbf{\tilde{x}}_m)|\nonumber\\
&=&\frac{1}{m}\Big|\Big|\mathbf{I} + \frac{\mathbf{\tilde{x}}_i
\mathbf{\tilde{x}}_i^\top}{\sigma^2}\Big|^{1/d}-\Big|\mathbf{I} +
\frac{\mathbf{\bar{x}}_i
\mathbf{\bar{x}}_i^\top}{\sigma^2}\Big|^{1/d}\Big|\nonumber\\
&\leq& \frac{1}{m} (\sqrt[d]{(R/\sigma)^2 +1} - 1) . \nonumber
\end{eqnarray}

By McDiarmid's inequality~\citep{JST04book}, we have for all
$\epsilon>0$,
\begin{equation}
P\left\{\mathbb{E}\Big[\Big|\mathbf{I} + \frac{\mathbf{\tilde{x}}
\mathbf{\tilde{x}}^\top}{\sigma^2}\Big|^{1/d}\Big] \geq
f(\mathbf{\tilde{x}}_1,\ldots,\mathbf{\tilde{x}}_m) -
\epsilon\right\} \geq
1-\exp\left(\frac{-2m\epsilon^2}{(\sqrt[d]{(R/\sigma)^2 +1} - 1)^2
}\right). \nonumber
\end{equation}
Setting the right hand size equal to $1-\frac{\delta}{3}$, we have
with probability at least $1-\frac{\delta}{3}$,
\begin{equation}
\mathbb{E}\Big[\Big|\mathbf{I} + \frac{\mathbf{\tilde{x}}
\mathbf{\tilde{x}}^\top}{\sigma^2}\Big|^{1/d}\Big] \geq
f(\mathbf{\tilde{x}}_1,\ldots,\mathbf{\tilde{x}}_m) -
(\sqrt[d]{(R/\sigma)^2 +1} - 1)
\sqrt{\frac{1}{2m}\ln\frac{3}{\delta}} , \nonumber
\end{equation}
and
\begin{equation}
-\ln\Big|\mathbf{I} + \frac{\mathbb{E}(\mathbf{\tilde{x}}
\mathbf{\tilde{x}}^\top)}{\sigma^2}\Big| \leq -d \ln\Big[
f(\mathbf{\tilde{x}}_1,\ldots,\mathbf{\tilde{x}}_m) -
(\sqrt[d]{(R/\sigma)^2 +1} - 1)
\sqrt{\frac{1}{2m}\ln\frac{3}{\delta}}~\Big]_+ , \label{eqn11021}
\end{equation}
where to reach (\ref{eqn11021}) we have used (\ref{eqn11022}) and
defined $[\cdot]_+=\max(\cdot,0)$.

Denote $H_m=\frac{1}{m}\sum_{i=1}^m [\mathbf{\tilde{x}}_i^\top
\mathbf{\tilde{x}}_i + \mu^2
(\mathbf{w}^\top\mathbf{\tilde{x}}_i)^2]$. It is clear that
\begin{equation}
\mathbb{E}[H_m]= \mathbb{E} \left\{\frac{1}{m}\sum_{i=1}^m
[\mathbf{\tilde{x}}_i^\top \mathbf{\tilde{x}}_i + \mu^2
(\mathbf{w}^\top\mathbf{\tilde{x}}_i)^2]\right\} =\mathbb{E}
[\mathbf{\tilde{x}}^\top \mathbf{\tilde{x}} + \mu^2
(\mathbf{w}^\top\mathbf{\tilde{x}})^2] . \nonumber
\end{equation}
Recall $R=\sup_\mathbf{\tilde{x}}\|\mathbf{\tilde{x}}\|$. By
McDiarmid's inequality, we have for all $\epsilon>0$,
\begin{equation}
P\left\{\mathbb{E}[H_m]\leq H_m + \epsilon\right\} \geq
1-\exp\left(\frac{-2m\epsilon^2}{(1+\mu^2)^2 R^4}\right). \nonumber
\end{equation}
Setting the right hand size equal to $1-\frac{\delta}{3}$, we have
with probability at least $1-\frac{\delta}{3}$,
\begin{equation}
\mathbb{E}[H_m]\leq H_m + (1+\mu^2) R^2
\sqrt{\frac{1}{2m}\ln\frac{3}{\delta}} . \label{eqn11024}
\end{equation}

In addition, from Lemma~\ref{LemLangford}, we have
\begin{equation}
Pr_{S\sim \mathcal{D}^m}\left(\forall
Q(c):KL_+(\hat{E}_{Q,S}||E_{Q,\mathcal{D}})\leq
\frac{KL(Q||P)+\ln\big(\frac{m+1}{\delta/3}\big)}{m}\right)\geq
1-\delta/3 . \label{eqn140201}
\end{equation}

According to the union bound ($Pr(A~or~B~or~C)\leq Pr(A) +
Pr(B)+Pr(C)$), the probability that at least one of the inequalities
in (\ref{eqn11021}), (\ref{eqn11024}) and (\ref{eqn140201}) fails is
no larger than $\delta/3 + \delta/3 +\delta/3 =\delta$. Hence, the
probability that all of the three inequalities hold is no less than
$1-\delta$. That is, with probability at least $1-\delta$ over
$S\sim \mathcal{D}^m$, the following inequality holds
\begin{eqnarray}
&&\forall \mathbf{w}, \mu: KL_+(\hat{E}_{Q,S}||E_{Q,\mathcal{D}})\leq\nonumber\\
&&\frac{-\frac{d}{2} \ln\Big[ f_m - (\sqrt[d]{(R/\sigma)^2 +1} - 1)
\sqrt{\frac{1}{2m}\ln\frac{3}{\delta}}~\Big]_+
+\frac{H_m}{2\sigma^2}
 + \frac{(1+\mu^2) R^2}{2\sigma^2}
\sqrt{\frac{1}{2m}\ln\frac{3}{\delta}} + \frac{\mu^2}{2}
+\ln\big(\frac{m+1}{\delta/3}\big)}{m},\nonumber
\end{eqnarray}
where $f_m$ is a shorthand for
$f(\mathbf{\tilde{x}}_1,\ldots,\mathbf{\tilde{x}}_m)$, and
$\|\mathbf{w}\|=1$.

\section{Proof of Theorem~\ref{thmMvPB2}}
\label{appB2}

Now the KL divergence between the posterior and prior becomes
\begin{eqnarray}
KL(Q(\mathbf{u})\|P(\mathbf{u})) &=& \frac{1}{2}\left(-\ln
(\Big|\mathbf{I} + \frac{\mathbb{E}(\mathbf{\tilde{x}}
\mathbf{\tilde{x}}^\top)}{\sigma^2}\Big|) + \frac{1}{\sigma^2}
\mathbb{E} [\mathbf{\tilde{x}}^\top \mathbf{\tilde{x}} + \mu^2
(\mathbf{w}^\top\mathbf{\tilde{x}})^2]+ \mu^2\right) \nonumber\\
&\leq& \frac{1}{2}\left(-\mathbb{E}\ln\Big|\mathbf{I} +
\frac{\mathbf{\tilde{x}} \mathbf{\tilde{x}}^\top}{\sigma^2}\Big| +
\frac{1}{\sigma^2} \mathbb{E} [\mathbf{\tilde{x}}^\top
\mathbf{\tilde{x}} + \mu^2 (\mathbf{w}^\top\mathbf{\tilde{x}})^2]+
\mu^2\right) \nonumber\\
&=&\frac{1}{2}\left(
 \mathbb{E}\Big(\frac{1}{\sigma^2} [\mathbf{\tilde{x}}^\top
\mathbf{\tilde{x}} + \mu^2
(\mathbf{w}^\top\mathbf{\tilde{x}})^2]-\ln\Big|\mathbf{I} +
\frac{\mathbf{\tilde{x}}
\mathbf{\tilde{x}}^\top}{\sigma^2}\Big|\Big)+ \mu^2\right).
\nonumber
\end{eqnarray}

Define
\begin{equation}
\tilde{f}(\mathbf{\tilde{x}}_1,\ldots,\mathbf{\tilde{x}}_m)=\frac{1}{m}\sum_{i=1}^m
\Big( \frac{1}{\sigma^2} [\mathbf{\tilde{x}}_i^\top
\mathbf{\tilde{x}}_i + \mu^2
(\mathbf{w}^\top\mathbf{\tilde{x}}_i)^2]-\ln\Big|\mathbf{I} +
\frac{\mathbf{\tilde{x}}_i \mathbf{\tilde{x}}_i^\top}{\sigma^2}\Big|
\Big) . \nonumber
\end{equation}
Recall $R=\sup_\mathbf{\tilde{x}}\|\mathbf{\tilde{x}}\|$. Since the
rank of matrix $\mathbf{\tilde{x}}_i
\mathbf{\tilde{x}}_i^\top/\sigma^2$ is 1 with the nonzero eigenvalue
being $\|\mathbf{\tilde{x}}_i\|^2/\sigma^2$ and the determinant of a
positive semi-definite matrix is equal to the product of its
eigenvalues, it follows that
\begin{eqnarray}
&&\sup_{\mathbf{\tilde{x}}_1,\ldots,\mathbf{\tilde{x}}_m,
\mathbf{\bar{x}}_i}|\tilde{f}(\mathbf{\tilde{x}}_1,\ldots,\mathbf{\tilde{x}}_m)-\tilde{f}(\mathbf{\tilde{x}}_1,\ldots,
\mathbf{\bar{x}}_i, \mathbf{\tilde{x}}_{i+1},\ldots,
\mathbf{\tilde{x}}_m)|\nonumber\\
&\leq& \frac{1}{m} \left(\frac{(1+\mu^2)R^2}{\sigma^2}+ \ln(1+
\frac{R^2}{\sigma^2} ) \right) . \nonumber
\end{eqnarray}

By McDiarmid's inequality, we have for all $\epsilon>0$,
\begin{equation}
P\left\{ \mathbb{E}\Big(\frac{1}{\sigma^2} [\mathbf{\tilde{x}}^\top
\mathbf{\tilde{x}} + \mu^2
(\mathbf{w}^\top\mathbf{\tilde{x}})^2]-\ln\Big|\mathbf{I} +
\frac{\mathbf{\tilde{x}}
\mathbf{\tilde{x}}^\top}{\sigma^2}\Big|\Big) \leq \tilde{f} +
\epsilon\right\} \geq 1-\exp\left(\frac{-2m\epsilon^2}{\Delta^2
}\right), \label{eqn11025}
\end{equation}
where $\tilde{f}$ is short for
$\tilde{f}(\mathbf{\tilde{x}}_1,\ldots,\mathbf{\tilde{x}}_m)$, and
$\Delta=\frac{(1+\mu^2)R^2}{\sigma^2}+ \ln(1+ \frac{R^2}{\sigma^2}
)$. Setting the right hand size of (\ref{eqn11025}) equal to
$1-\frac{\delta}{2}$, we have with probability at least
$1-\frac{\delta}{2}$,
\begin{equation}
\mathbb{E}\Big(\frac{1}{\sigma^2} [\mathbf{\tilde{x}}^\top
\mathbf{\tilde{x}} + \mu^2
(\mathbf{w}^\top\mathbf{\tilde{x}})^2]-\ln\Big|\mathbf{I} +
\frac{\mathbf{\tilde{x}}
\mathbf{\tilde{x}}^\top}{\sigma^2}\Big|\Big) \leq \tilde{f} +
\Delta\sqrt{\frac{1}{2m}\ln\frac{2}{\delta}}. \nonumber
\end{equation}

Meanwhile, from Lemma~\ref{LemLangford}, we have
\begin{equation}
 Pr_{S\sim \mathcal{D}^m}\left(\forall
Q(c):KL_+(\hat{E}_{Q,S}||E_{Q,\mathcal{D}})\leq
\frac{KL(Q||P)+\ln\big(\frac{m+1}{\delta/2}\big)}{m}\right)\geq
1-\delta/2 . \nonumber
\end{equation}

According to the union bound, we can complete the proof for the
dimensionality-independent PAC-Bayes bound.

\section{Proof of Theorem~\ref{thmMvPB3}}
\label{appB3}

It is clear that from
$R=\sup_{\mathbf{\tilde{x}}}\|\mathbf{\tilde{x}}\|$, we have
$\sup_{\mathbf{x}}\|\mathbf{x}\|=R$ and
$\sup_{(\mathbf{x},y)}\|y\mathbf{x}\|=R$.

From (\ref{eqn11021}), it follows that
 with probability at least $1-\frac{\delta}{4}$,
\begin{equation}
-\ln\Big|\mathbf{I} + \frac{\mathbb{E}(\mathbf{\tilde{x}}
\mathbf{\tilde{x}}^\top)}{\sigma^2}\Big| \leq -d \ln\Big[ f_m -
(\sqrt[d]{(R/\sigma)^2 +1} - 1)
\sqrt{\frac{1}{2m}\ln\frac{4}{\delta}}~\Big]_+ . \nonumber
\end{equation}

With reference to a bounding result on estimating the center of
mass~\citep{JST04book}, it follows that with probability at least
$1-\delta/4$ the following inequality holds
\begin{equation}
\|\mathbf{w}_p-\mathbf{\hat{w}}_p\|\leq
\frac{R}{\sqrt{m}}\left(2+\sqrt{2\ln\frac{4}{\delta}}\right) .
\nonumber
\end{equation}

Denote $\hat{H}_m=\frac{1}{m}\sum_{i=1}^m [\mathbf{\tilde{x}}_i^\top
\mathbf{\tilde{x}}_i -2\eta\mu \sigma^2
y_i(\mathbf{w}^\top\mathbf{x}_i)+ \mu^2
(\mathbf{w}^\top\mathbf{\tilde{x}}_i)^2]$. It is clear that
\begin{equation}
\mathbb{E}[\hat{H}_m]= \mathbb{E} [\mathbf{\tilde{x}}^\top
\mathbf{\tilde{x}} -2\eta\mu \sigma^2 y(\mathbf{w}^\top\mathbf{x})+
\mu^2 (\mathbf{w}^\top\mathbf{\tilde{x}})^2] . \nonumber
\end{equation}
By McDiarmid's inequality, we have for all $\epsilon>0$,
\begin{equation}
P\left\{\mathbb{E}[\hat{H}_m]\leq \hat{H}_m + \epsilon\right\} \geq
1-\exp\left(\frac{-2m\epsilon^2}{(R^2+4\eta\mu\sigma^2 R+ \mu^2
R^2)^2}\right). \nonumber
\end{equation}
Setting the right hand size equal to $1-\frac{\delta}{4}$, we have
with probability at least $1-\frac{\delta}{4}$,
\begin{equation}
\mathbb{E}[\hat{H}_m]\leq \hat{H}_m + (R^2+ \mu^2
R^2+4\eta\mu\sigma^2 R) \sqrt{\frac{1}{2m}\ln\frac{4}{\delta}} .
\nonumber
\end{equation}

In addition, according to Lemma~\ref{LemLangford}, we have
\begin{equation}
Pr_{S\sim \mathcal{D}^m}\left(\forall
Q(c):KL_+(\hat{E}_{Q,S}||E_{Q,\mathcal{D}})\leq
\frac{KL(Q||P)+\ln\big(\frac{m+1}{\delta/4}\big)}{m}\right)\geq
1-\delta/4 . \nonumber
\end{equation}

Therefore, from the union bound, we get the result.

\section{Proof of Theorem~\ref{thmMvPB4}}
\label{appB4}

Applying (\ref{eqn11023}) to (\ref{eqn7}), we obtain
\begin{eqnarray}
KL(Q(\mathbf{u})\|P(\mathbf{u})) &\leq& -\frac{1}{2}\mathbb{E}\ln
\Big|\mathbf{I} + \frac{\mathbf{\tilde{x}}
\mathbf{\tilde{x}}^\top}{\sigma^2}\Big|
+\frac{1}{2}(\|\eta\mathbf{w}_p-
\eta\mathbf{\hat{w}}_p\|+\|\eta\mathbf{\hat{w}}_p-\mu\mathbf{w} \| + \mu)^2 + \nonumber\\
&& \frac{1}{2\sigma^2}\mathbb{E}\left[ \mathbf{\tilde{x}}^\top
\mathbf{\tilde{x}}  -2\eta\mu \sigma^2 y(\mathbf{w}^\top\mathbf{x})
+ \mu^2 (\mathbf{w}^\top\mathbf{\tilde{x}})^2 \right] +\frac{\mu^2
}{2}\nonumber\\
&=&\frac{1}{2}(\|\eta\mathbf{w}_p-
\eta\mathbf{\hat{w}}_p\|+\|\eta\mathbf{\hat{w}}_p-\mu\mathbf{w} \| + \mu)^2 + \nonumber\\
&& \frac{1}{2}\mathbb{E}\left[ \frac{\mathbf{\tilde{x}}^\top
\mathbf{\tilde{x}}  -2\eta\mu \sigma^2 y(\mathbf{w}^\top\mathbf{x})
+ \mu^2 (\mathbf{w}^\top\mathbf{\tilde{x}})^2 }{\sigma^2} -\ln
\Big|\mathbf{I} + \frac{\mathbf{\tilde{x}}
\mathbf{\tilde{x}}^\top}{\sigma^2}\Big|\right] +\frac{\mu^2 }{2} .
\nonumber
\end{eqnarray}

Following~\citet{JST04book}, we have with probability at least
$1-\delta/3$
\begin{equation}
\|\mathbf{w}_p-\mathbf{\hat{w}}_p\|\leq
\frac{R}{\sqrt{m}}\left(2+\sqrt{2\ln\frac{3}{\delta}}\right) .
\nonumber
\end{equation}

Denote $\tilde{H}_m=\frac{1}{m}\sum_{i=1}^m
[\frac{\mathbf{\tilde{x}}_i^\top \mathbf{\tilde{x}}_i -2\eta\mu
\sigma^2 y_i(\mathbf{w}^\top\mathbf{x}_i)+ \mu^2
(\mathbf{w}^\top\mathbf{\tilde{x}}_i)^2}{\sigma^2} - \ln
\Big|\mathbf{I} + \frac{\mathbf{\tilde{x}}_i
\mathbf{\tilde{x}}_i^\top}{\sigma^2}\Big|]$. It is clear that
\begin{equation}
\mathbb{E}[\tilde{H}_m]= \mathbb{E} [\frac{\mathbf{\tilde{x}}^\top
\mathbf{\tilde{x}} -2\eta\mu \sigma^2 y(\mathbf{w}^\top\mathbf{x})+
\mu^2 (\mathbf{w}^\top\mathbf{\tilde{x}})^2}{\sigma^2} - \ln
\Big|\mathbf{I} + \frac{\mathbf{\tilde{x}}
\mathbf{\tilde{x}}^\top}{\sigma^2}\Big|] . \nonumber
\end{equation}
By McDiarmid's inequality, we have for all $\epsilon>0$,
\begin{equation}
P\left\{\mathbb{E}[\tilde{H}_m]\leq \tilde{H}_m + \epsilon\right\}
\geq
1-\exp\left(\frac{-2m\epsilon^2}{\big(\frac{R^2+4\eta\mu\sigma^2 R+
\mu^2 R^2}{\sigma^2} +\ln(1+\frac{R^2}{\sigma^2})\big)^2}\right).
\nonumber
\end{equation}
Setting the right hand size equal to $1-\frac{\delta}{3}$, we have
with probability at least $1-\frac{\delta}{3}$,
\begin{equation}
\mathbb{E}[\tilde{H}_m]\leq \tilde{H}_m +
\big(\frac{R^2+4\eta\mu\sigma^2 R+ \mu^2 R^2}{\sigma^2}
+\ln(1+\frac{R^2}{\sigma^2})\big)
\sqrt{\frac{1}{2m}\ln\frac{3}{\delta}} . \nonumber
\end{equation}

In addition,  from Lemma~\ref{LemLangford}, we have
\begin{equation}
Pr_{S\sim \mathcal{D}^m}\left(\forall
Q(c):KL_+(\hat{E}_{Q,S}||E_{Q,\mathcal{D}})\leq
\frac{KL(Q||P)+\ln\big(\frac{m+1}{\delta/3}\big)}{m}\right)\geq
1-\delta/3 . \nonumber
\end{equation}

By applying the union bound, we complete the proof.

\section{Proof of Theorem~\ref{thmSMvPB2}}
\label{appSB2}

We already have $\sup_{\mathbf{x}}\|\mathbf{x}\|=R$ and
$\sup_{(\mathbf{x},y)}\|y\mathbf{x}\|=R$ from the definition
$R=\sup_{\mathbf{\tilde{x}}}\|\mathbf{\tilde{x}}\|$.

Following~\citet{JST04book}, we have with probability at least
$1-\delta/3$
\begin{equation}
\|\mathbf{w}_p-\mathbf{\hat{w}}_p\|\leq
\frac{R}{\sqrt{m}}\left(2+\sqrt{2\ln\frac{3}{\delta}}\right) .
\nonumber
\end{equation}

Denote $\bar{S}_m=\frac{1}{m}\sum_{i=1}^m [-\eta\mu
y_i(\mathbf{w}^\top\mathbf{x}_i)]$. It is clear that
\begin{equation}
\mathbb{E}[\bar{S}_m]= -\eta\mu \mathbb{E}\left[
y(\mathbf{w}^\top\mathbf{x})\right]. \nonumber
\end{equation}
By McDiarmid's inequality, we have for all $\epsilon>0$,
\begin{equation}
P\left\{\mathbb{E}[\bar{S}_m]\leq \bar{S}_m + \epsilon\right\} \geq
1-\exp\left(\frac{-2m\epsilon^2}{\big(2\eta\mu R\big)^2}\right).
\nonumber
\end{equation}
Setting the right hand size equal to $1-\frac{\delta}{3}$, we have
with probability at least $1-\frac{\delta}{3}$,
\begin{equation}
\mathbb{E}[\bar{S}_m]\leq \bar{S}_m + \eta\mu R
\sqrt{\frac{2}{m}\ln\frac{3}{\delta}} . \nonumber
\end{equation}

In addition,  from Lemma~\ref{LemLangford}, we have
\begin{equation}
Pr_{S\sim \mathcal{D}^m}\left(\forall
Q(c):KL_+(\hat{E}_{Q,S}||E_{Q,\mathcal{D}})\leq
\frac{KL(Q||P)+\ln\big(\frac{m+1}{\delta/3}\big)}{m}\right)\geq
1-\delta/3 . \nonumber
\end{equation}

After applying the union bound, the proof is completed.

\section{Dual Optimization Derivation for MvSVMs}
\label{appMvSVM}

To optimize (\ref{eqnOptProblem}), here we derive the Lagrange dual
function.

Let $\lambda_1^i, \lambda_2^i, \nu_1^i, \nu_2^i \geq 0$ be the
Lagrange multipliers associated with the inequality constraints of
problem (\ref{eqnOptProblem}). The Lagrangian
${L}(\boldsymbol{\alpha}_1,\boldsymbol{\alpha}_2,
\boldsymbol{\xi}_1, \boldsymbol{\xi}_2, \boldsymbol{\lambda}_1,
\boldsymbol{\lambda}_2, \boldsymbol{\nu}_1, \boldsymbol{\nu}_2)$ can
be written as
\begin{eqnarray}
{L}=&&F_0-\sum_{i=1}^{n}\Big[\lambda_1^i\Big(y_{i}(\sum_{j=1}^{n}\alpha_{1}^{j}k_1(x_{j},x_{i}))-
1+\xi_1^{i}\Big)+\nonumber\\
&&\lambda_2^i\Big(y_{i}(\sum_{j=1}^{n}\alpha_{2}^{j}k_2(x_{j},x_{i}))-
1+\xi_2^{i}\Big) + \nu_1^i\xi_1^{i}+ \nu_2^i\xi_2^{i}\Big].
\nonumber
\end{eqnarray}

To obtain the Lagrangian dual function, ${L}$ has to be minimized
with respect to the primal variables
$\boldsymbol{\alpha}_1,\boldsymbol{\alpha}_2, \boldsymbol{\xi}_1,
\boldsymbol{\xi}_2$. To eliminate these variables, we compute the
corresponding partial derivatives and set them to $0$, obtaining the
following conditions
\begin{eqnarray}
\label{eqn05151} && (K_1 +2C_2 K_1K_1) \boldsymbol{\alpha}_1-2C_2
K_1 K_2 \boldsymbol{\alpha}_2  =\Lambda_1,\\
\label{eqn05153} &&(K_2 +2C_2 K_2K_2) \boldsymbol{\alpha}_2-2C_2
K_2 K_1 \boldsymbol{\alpha}_1 =\Lambda_2,\\
\label{eqn05166}
&&\lambda_1^i + \nu_1^i =C_1, \\
\label{eqn05167} &&\lambda_2^i + \nu_2^i =C_1,
\end{eqnarray}
where we have defined
\begin{eqnarray}
\Lambda_1      &\triangleq& \sum_{i=1}^{n}\lambda_1^i y_{i}K_1(:,i),\nonumber\\
\Lambda_2      &\triangleq& \sum_{i=1}^{n}\lambda_2^i
y_{i}K_2(:,i),\nonumber
\end{eqnarray}
with $K_1(:,i)$ and $K_2(:,i)$ being the $i$th columns of the
corresponding Gram matrices.

Substituting~(\ref{eqn05151})$\sim$(\ref{eqn05167}) into ${L}$
results in the following expression of the Lagrangian dual function
${g}(\boldsymbol{\lambda}_1, \boldsymbol{\lambda}_2,
\boldsymbol{\nu}_1, \boldsymbol{\nu}_2)$
\begin{eqnarray}
{g}&=&\frac{1}{2}({\boldsymbol{\alpha}}_1^{\top}K_1\boldsymbol{\alpha}_1+
{\boldsymbol{\alpha}}_2^{\top}K_2\boldsymbol{\alpha}_2)+C_2
(\boldsymbol{\alpha}_1^\top K_1 K_1 \boldsymbol{\alpha}_1 -
2\boldsymbol{\alpha}_1^\top K_1 K_2 \boldsymbol{\alpha}_2 +
\nonumber\\
&&\boldsymbol{\alpha}_2^\top K_2 K_2 \boldsymbol{\alpha}_2)-
\boldsymbol{\alpha}_1^{\top}\Lambda_1-\boldsymbol{\alpha}_2^{\top}\Lambda_2+
\sum_{i=1}^{n}(\lambda_1^i+\lambda_2^i)\nonumber\\
&=&\frac{1}{2}\boldsymbol{\alpha}_1^\top \Lambda_1+
\frac{1}{2}\boldsymbol{\alpha}_2^\top\Lambda_2 -
\boldsymbol{\alpha}_1^{\top}\Lambda_1-\boldsymbol{\alpha}_2^{\top}\Lambda_2+
\sum_{i=1}^{n}(\lambda_1^i+\lambda_2^i) \nonumber\\
&=&-\frac{1}{2}\boldsymbol{\alpha}_1^\top \Lambda_1 -
\frac{1}{2}\boldsymbol{\alpha}_2^\top
\Lambda_2+\sum_{i=1}^{n}(\lambda_1^i+\lambda_2^i). \label{eqn05165}
\end{eqnarray}

Define
\begin{eqnarray}
&& \tilde{K}_1 = K_1 +2C_2 K_1K_1, \quad \bar{K}_1=2C_2
K_1 K_2, \nonumber\\
&& \tilde{K}_2 = K_2 +2C_2 K_2K_2, \quad \bar{K}_2=2C_2 K_2 K_1.
\nonumber
\end{eqnarray}
Then,  (\ref{eqn05151}) and (\ref{eqn05153}) become
\begin{eqnarray}
\label{eqn1310081} && \tilde{K}_1 \boldsymbol{\alpha}_1-
\bar{K}_1  \boldsymbol{\alpha}_2 =\Lambda_1,\\
\label{eqn1310082} && \tilde{K}_2 \boldsymbol{\alpha}_2- \bar{K}_2
\boldsymbol{\alpha}_1 =\Lambda_2.
\end{eqnarray}

From (\ref{eqn1310081}) and (\ref{eqn1310082}), we have
\begin{eqnarray}
&&(\tilde{K}_1-\bar{K}_1\tilde{K}_2^{-1}\bar{K}_2)\boldsymbol{\alpha}_1=\bar{K}_1\tilde{K}_2^{-1}
\Lambda_2 +\Lambda_1 \nonumber\\
&&(\tilde{K}_2-\bar{K}_2\tilde{K}_1^{-1}\bar{K}_1)\boldsymbol{\alpha}_2=\bar{K}_2\tilde{K}_1^{-1}
\Lambda_1 +\Lambda_2. \nonumber
\end{eqnarray}
Define $M_1\triangleq
\tilde{K}_1-\bar{K}_1\tilde{K}_2^{-1}\bar{K}_2$ and $M_2\triangleq
\tilde{K}_2-\bar{K}_2\tilde{K}_1^{-1}\bar{K}_1$. It follows that
\begin{eqnarray}
&&\boldsymbol{\alpha}_1=M_1^{-1}\Big[\bar{K}_1\tilde{K}_2^{-1}
\Lambda_2 +\Lambda_1\Big], \label{eqnalpha1} \\
&&\boldsymbol{\alpha}_2=M_2^{-1}\Big[\bar{K}_2\tilde{K}_1^{-1}
\Lambda_1 +\Lambda_2\Big]. \label{eqnalpha2}
\end{eqnarray}

Now with $\boldsymbol{\alpha}_1$ and $\boldsymbol{\alpha}_2$
substituted into~(\ref{eqn05165}), the Lagrange dual function
${g}(\boldsymbol{\lambda}_1, \boldsymbol{\lambda}_2,
\boldsymbol{\nu}_1, \boldsymbol{\nu}_2)$ is
\begin{eqnarray}
{g}&=&\inf_{\boldsymbol{\alpha}_1,\boldsymbol{\alpha}_2,
\boldsymbol{\xi}_1, \boldsymbol{\xi}_2}
{L}=-\frac{1}{2}\boldsymbol{\alpha}_1^\top \Lambda_1-
\frac{1}{2}\boldsymbol{\alpha}_2^\top \Lambda_2+\sum_{i=1}^{n}(\lambda_1^i+\lambda_2^i)\nonumber\\
&=&-\frac{1}{2}\Lambda_1^{\top} M_1^{-1}
\Big[\bar{K}_1\tilde{K}_2^{-1} \Lambda_2 +
\Lambda_1\Big]-\frac{1}{2}\Lambda_2^{\top}M_2^{-1}
\Big[\bar{K}_2\tilde{K}_1^{-1} \Lambda_1 + \Lambda_2\Big]+
\sum_{i=1}^{n}(\lambda_1^i+\lambda_2^i). \nonumber
\end{eqnarray}

The Lagrange dual problem is given by
\begin{eqnarray}
\max_{\boldsymbol{\lambda}_1, \boldsymbol{\lambda}_2} &&\; {g}\nonumber\\
\mbox{s.t.} && \left\{
\begin{array}{ll}
0\leq \lambda_{1}^{i}\leq C_1, \quad&i=1,\ldots,n\\
0\leq \lambda_{2}^{i}\leq C_1, \quad&i=1,\ldots,n .
\end{array}
\right. %\nonumber
\end{eqnarray}

As Lagrange dual functions are concave, we can formulate the
Lagrange dual problem as a convex optimization problem
\begin{eqnarray}
\label{eqn05171}
\min_{\boldsymbol{\lambda}_1, \boldsymbol{\lambda}_2} &&-{g}\nonumber\\
\mbox{s.t.} && \left\{
\begin{array}{ll}
0\leq \lambda_{1}^{i}\leq C_1, \quad&i=1,\ldots,n\\
0\leq \lambda_{2}^{i}\leq C_1, \quad&i=1,\ldots,n .
\end{array}
\right. %\nonumber
\end{eqnarray}

Define matrix $Y\triangleq \mbox{diag}(y_{1},\ldots,y_{n})$. Then,
$\Lambda_1=K_{1} Y\boldsymbol{\lambda}_1$ and $\Lambda_2=K_{2}
Y\boldsymbol{\lambda}_2$ with
$\boldsymbol{\lambda}_1=(\lambda_{1}^{1},...,\lambda_{1}^{n})^\top$,
and
$\boldsymbol{\lambda}_2=(\lambda_{2}^{1},...,\lambda_{2}^{n})^\top$.
It is clear that $\tilde{K}_1$ and $\tilde{K}_2$ are symmetric
matrices, and $\bar{K}_1=\bar{K}_2^\top$. Therefore, it follows that
matrices $M_1$ and $M_2$ are also symmetric.

We have
\begin{eqnarray}
-{g}&=&\frac{1}{2}\Lambda_1^{\top} M_1^{-1}
\Big[\bar{K}_1\tilde{K}_2^{-1} \Lambda_2
+\Lambda_1\Big]+\frac{1}{2}\Lambda_2^{\top}M_2^{-1}
\Big[\bar{K}_2\tilde{K}_1^{-1}\Lambda_1 + \Lambda_2\Big]-
\sum_{i=1}^{n}(\lambda_1^i+\lambda_2^i) \nonumber\\
&=&\frac{1}{2}\Big\{\bm\lambda_1^\top [YK_1M_1^{-1}K_1Y]\bm\lambda_1
+\bm\lambda_1^\top
[YK_1M_1^{-1}\bar{K}_1\tilde{K}_2^{-1}K_2Y]\bm\lambda_2 + \nonumber\\
&& \bm\lambda_2^\top
[YK_2M_2^{-1}\bar{K}_2\tilde{K}_1^{-1}K_1Y]\bm\lambda_1 +
\bm\lambda_2^\top [YK_2M_2^{-1}K_2Y]\bm\lambda_2 \Big\}
-\textbf{1}^\top
(\boldsymbol{\lambda}_1+\boldsymbol{\lambda}_2)\nonumber\\
&=&\frac{1}{2}(\boldsymbol{\lambda}_1^\top
\;\boldsymbol{\lambda}_2^\top) {\left(
\begin{array}{cc}
A&\; B\\
B^\top&\; D
\end{array}
\right)} {\left(
\begin{array}{c}
\boldsymbol{\lambda}_1\\
\boldsymbol{\lambda}_2
\end{array}
\right)} - {\left(
\begin{array}{c}
\boldsymbol{\lambda}_1\\
\boldsymbol{\lambda}_2
\end{array}
\right)^\top \textbf{1}_{2n}}, \nonumber
\end{eqnarray}
where
\begin{equation}
A\triangleq YK_1M_1^{-1}K_1Y, \quad
 B\triangleq
YK_1M_1^{-1}\bar{K}_1\tilde{K}_2^{-1}K_2Y, \quad D\triangleq
YK_2M_2^{-1}K_2Y , \label{eqn07031}
\end{equation}
$\textbf{1}_{2n}=(1,\ldots,1_{(2n)})^\top$, and we have used the
fact that
\begin{equation}
YK_1M_1^{-1}\bar{K}_1\tilde{K}_2^{-1}K_2Y
=[YK_2M_2^{-1}\bar{K}_2\tilde{K}_1^{-1}K_1Y]^\top. \nonumber
\end{equation}
Because of the convexity of function $-{g}$, we affirm that matrix
$\left(
\begin{array}{cc}
A&\; B\\
B^\top&\; D
\end{array}
\right)$ is positive semidefinite.

Hence, the optimization problem in (\ref{eqn05171}) can be rewritten
as
\begin{eqnarray}
\min_{\boldsymbol{\lambda}_1, \boldsymbol{\lambda}_2}
&&\frac{1}{2}(\boldsymbol{\lambda}_1^\top
\;\boldsymbol{\lambda}_2^\top) {\left(
\begin{array}{cc}
A&\; B\\
B^\top&\; D
\end{array}
\right)} {\left(
\begin{array}{c}
\boldsymbol{\lambda}_1\\
\boldsymbol{\lambda}_2
\end{array}
\right)} - {\left(
\begin{array}{c}
\boldsymbol{\lambda}_1\\
\boldsymbol{\lambda}_2
\end{array}
\right)^\top} \mathbf{1}_{2n} \nonumber\\
\mbox{s.t.} && \left\{
\begin{array}{l}
0\preceq \boldsymbol{\lambda}_{1}\preceq C_1 \textbf{1}, \\
0\preceq \boldsymbol{\lambda}_{2}\preceq C_1 \textbf{1}.
\end{array}
\right. \label{eqn07032}
\end{eqnarray}

After solving this problem, we can then obtain classifier parameters
$\boldsymbol{\alpha}_{1}$ and $\boldsymbol{\alpha}_{2}$ using
(\ref{eqnalpha1}) and (\ref{eqnalpha2}), which are finally used by
(\ref{eqn1310111}).

\section{Dual Optimization Derivation for SMvSVMs}
\label{appSMvSVM}

To optimize (\ref{eqnOptProblemS}),  we first derive the Lagrange
dual function following the same line of optimization derivations
for MvSVMs. Although here some of the derivations are similar to
those for MvSVMs, for completeness we include them.

Let $\lambda_1^i, \lambda_2^i, \nu_1^i, \nu_2^i \geq 0$ be the
Lagrange multipliers associated with the inequality constraints of
problem (\ref{eqnOptProblemS}). The Lagrangian
${L}(\boldsymbol{\alpha}_1,\boldsymbol{\alpha}_2,
\boldsymbol{\xi}_1, \boldsymbol{\xi}_2, \boldsymbol{\lambda}_1,
\boldsymbol{\lambda}_2, \boldsymbol{\nu}_1, \boldsymbol{\nu}_2)$ can
be formulated as
\begin{eqnarray}
{L}=&&\tilde{F}_0-\sum_{i=1}^{n}\Big[\lambda_1^i\Big(y_{i}(\sum_{j=1}^{n+u}\alpha_{1}^{j}k_1(x_{j},x_{i}))-
1+\xi_1^{i}\Big)+\nonumber\\
&&\lambda_2^i\Big(y_{i}(\sum_{j=1}^{n+u}\alpha_{2}^{j}k_2(x_{j},x_{i}))-
1+\xi_2^{i}\Big) + \nu_1^i\xi_1^{i}+ \nu_2^i\xi_2^{i}\Big].
\nonumber
\end{eqnarray}

To obtain the Lagrangian dual function, ${L}$ will be minimized with
respect to the primal variables
$\boldsymbol{\alpha}_1,\boldsymbol{\alpha}_2, \boldsymbol{\xi}_1,
\boldsymbol{\xi}_2$. To eliminate these variables, setting the
corresponding partial derivatives to $0$ results in the following
conditions
\begin{eqnarray}
\label{eqn05151S} && (K_1 +2C_2 K_1K_1) \boldsymbol{\alpha}_1-2C_2
K_1 K_2 \boldsymbol{\alpha}_2  =\Lambda_1,\\
\label{eqn05153S} &&(K_2 +2C_2 K_2K_2) \boldsymbol{\alpha}_2-2C_2
K_2 K_1 \boldsymbol{\alpha}_1 =\Lambda_2,\\
\label{eqn05166S}
&&\lambda_1^i + \nu_1^i =C_1, \\
\label{eqn05167S} &&\lambda_2^i + \nu_2^i =C_1,
\end{eqnarray}
where we have defined
\begin{eqnarray}
\Lambda_1      &\triangleq& \sum_{i=1}^{n}\lambda_1^i y_{i}K_1(:,i),\nonumber\\
\Lambda_2      &\triangleq& \sum_{i=1}^{n}\lambda_2^i
y_{i}K_2(:,i),\nonumber
\end{eqnarray}
with $K_1(:,i)$ and $K_2(:,i)$ being the $i$th columns of the
corresponding Gram matrices.

Substituting~(\ref{eqn05151S})$\sim$(\ref{eqn05167S}) into ${L}$
results in the Lagrangian dual function ${g}(\boldsymbol{\lambda}_1,
\boldsymbol{\lambda}_2, \boldsymbol{\nu}_1, \boldsymbol{\nu}_2)$
\begin{eqnarray}
{g}&=&\frac{1}{2}({\boldsymbol{\alpha}}_1^{\top}K_1\boldsymbol{\alpha}_1+
{\boldsymbol{\alpha}}_2^{\top}K_2\boldsymbol{\alpha}_2)+C_2
(\boldsymbol{\alpha}_1^\top K_1 K_1 \boldsymbol{\alpha}_1 -
2\boldsymbol{\alpha}_1^\top K_1 K_2 \boldsymbol{\alpha}_2 +
\nonumber\\
&&\boldsymbol{\alpha}_2^\top K_2 K_2 \boldsymbol{\alpha}_2)-
\boldsymbol{\alpha}_1^{\top}\Lambda_1-\boldsymbol{\alpha}_2^{\top}\Lambda_2+
\sum_{i=1}^{n}(\lambda_1^i+\lambda_2^i)\nonumber\\
&=&\frac{1}{2}\boldsymbol{\alpha}_1^\top \Lambda_1+
\frac{1}{2}\boldsymbol{\alpha}_2^\top\Lambda_2 -
\boldsymbol{\alpha}_1^{\top}\Lambda_1-\boldsymbol{\alpha}_2^{\top}\Lambda_2+
\sum_{i=1}^{n}(\lambda_1^i+\lambda_2^i) \nonumber\\
&=&-\frac{1}{2}\boldsymbol{\alpha}_1^\top \Lambda_1 -
\frac{1}{2}\boldsymbol{\alpha}_2^\top
\Lambda_2+\sum_{i=1}^{n}(\lambda_1^i+\lambda_2^i). \label{eqn05165S}
\end{eqnarray}

Define
\begin{eqnarray}
&& \tilde{K}_1 = K_1 +2C_2 K_1K_1, \quad \bar{K}_1=2C_2
K_1 K_2, \nonumber\\
&& \tilde{K}_2 = K_2 +2C_2 K_2K_2, \quad \bar{K}_2=2C_2 K_2 K_1.
\nonumber
\end{eqnarray}
Then,  (\ref{eqn05151S}) and (\ref{eqn05153S}) become
\begin{eqnarray}
\label{eqn1310081S} && \tilde{K}_1 \boldsymbol{\alpha}_1-
\bar{K}_1  \boldsymbol{\alpha}_2 =\Lambda_1,\\
\label{eqn1310082S} && \tilde{K}_2 \boldsymbol{\alpha}_2- \bar{K}_2
\boldsymbol{\alpha}_1 =\Lambda_2.
\end{eqnarray}

From (\ref{eqn1310081S}) and (\ref{eqn1310082S}), we have
\begin{eqnarray}
&&(\tilde{K}_1-\bar{K}_1\tilde{K}_2^{-1}\bar{K}_2)\boldsymbol{\alpha}_1=\bar{K}_1\tilde{K}_2^{-1}
\Lambda_2 +\Lambda_1 \nonumber\\
&&(\tilde{K}_2-\bar{K}_2\tilde{K}_1^{-1}\bar{K}_1)\boldsymbol{\alpha}_2=\bar{K}_2\tilde{K}_1^{-1}
\Lambda_1 +\Lambda_2. \nonumber
\end{eqnarray}
Define $M_1\triangleq
\tilde{K}_1-\bar{K}_1\tilde{K}_2^{-1}\bar{K}_2$ and $M_2\triangleq
\tilde{K}_2-\bar{K}_2\tilde{K}_1^{-1}\bar{K}_1$. It is clear that
\begin{eqnarray}
&&\boldsymbol{\alpha}_1=M_1^{-1}\Big[\bar{K}_1\tilde{K}_2^{-1}
\Lambda_2 +\Lambda_1\Big], \label{eqnalpha1S} \\
&&\boldsymbol{\alpha}_2=M_2^{-1}\Big[\bar{K}_2\tilde{K}_1^{-1}
\Lambda_1 +\Lambda_2\Big]. \label{eqnalpha2S}
\end{eqnarray}

With $\boldsymbol{\alpha}_1$ and $\boldsymbol{\alpha}_2$ substituted
into~(\ref{eqn05165S}), the Lagrange dual function
${g}(\boldsymbol{\lambda}_1, \boldsymbol{\lambda}_2,
\boldsymbol{\nu}_1, \boldsymbol{\nu}_2)$ is then
\begin{eqnarray}
{g}&=&\inf_{\boldsymbol{\alpha}_1,\boldsymbol{\alpha}_2,
\boldsymbol{\xi}_1, \boldsymbol{\xi}_2}
{L}=-\frac{1}{2}\boldsymbol{\alpha}_1^\top \Lambda_1-
\frac{1}{2}\boldsymbol{\alpha}_2^\top \Lambda_2+\sum_{i=1}^{n}(\lambda_1^i+\lambda_2^i)\nonumber\\
&=&-\frac{1}{2}\Lambda_1^{\top} M_1^{-1}
\Big[\bar{K}_1\tilde{K}_2^{-1} \Lambda_2 +
\Lambda_1\Big]-\frac{1}{2}\Lambda_2^{\top}M_2^{-1}
\Big[\bar{K}_2\tilde{K}_1^{-1} \Lambda_1 + \Lambda_2\Big]+
\sum_{i=1}^{n}(\lambda_1^i+\lambda_2^i). \nonumber
\end{eqnarray}

The Lagrange dual problem is given by
\begin{eqnarray}
\max_{\boldsymbol{\lambda}_1, \boldsymbol{\lambda}_2} &&\; {g}\nonumber\\
\mbox{s.t.} && \left\{
\begin{array}{ll}
0\leq \lambda_{1}^{i}\leq C_1, \quad&i=1,\ldots,n\\
0\leq \lambda_{2}^{i}\leq C_1, \quad&i=1,\ldots,n .
\end{array}
\right. %\nonumber
\end{eqnarray}

As Lagrange dual functions are concave, below we formulate the
Lagrange dual problem as a convex optimization problem
\begin{eqnarray}
\label{eqn05171S}
\min_{\boldsymbol{\lambda}_1, \boldsymbol{\lambda}_2} &&-{g}\nonumber\\
\mbox{s.t.} && \left\{
\begin{array}{ll}
0\leq \lambda_{1}^{i}\leq C_1, \quad&i=1,\ldots,n\\
0\leq \lambda_{2}^{i}\leq C_1, \quad&i=1,\ldots,n .
\end{array}
\right. %\nonumber
\end{eqnarray}

Define matrix $Y\triangleq \mbox{diag}(y_{1},\ldots,y_{n})$. Then,
$\Lambda_1=K_{n1} Y\boldsymbol{\lambda}_1$ and $\Lambda_2=K_{n2}
Y\boldsymbol{\lambda}_2$ with $K_{n1}=K_1(:,1:n)$,
$K_{n2}=K_2(:,1:n)$,
$\boldsymbol{\lambda}_1=(\lambda_{1}^{1},...,\lambda_{1}^{n})^\top$,
and
$\boldsymbol{\lambda}_2=(\lambda_{2}^{1},...,\lambda_{2}^{n})^\top$.
It is clear that $\tilde{K}_1$ and $\tilde{K}_2$ are symmetric
matrices, and $\bar{K}_1=\bar{K}_2^\top$. Therefore, it follows that
matrices $M_1$ and $M_2$ are also symmetric.

We have
\begin{eqnarray}
-{g}&=&\frac{1}{2}\Lambda_1^{\top} M_1^{-1}
\Big[\bar{K}_1\tilde{K}_2^{-1} \Lambda_2
+\Lambda_1\Big]+\frac{1}{2}\Lambda_2^{\top}M_2^{-1}
\Big[\bar{K}_2\tilde{K}_1^{-1}\Lambda_1 + \Lambda_2\Big]-
\sum_{i=1}^{n}(\lambda_1^i+\lambda_2^i) \nonumber\\
&=&\frac{1}{2}\Big\{\bm\lambda_1^\top [YK_{n1}^\top
M_1^{-1}K_{n1}Y]\bm\lambda_1 +\bm\lambda_1^\top
[YK_{n1}^\top M_1^{-1}\bar{K}_1\tilde{K}_2^{-1}K_{n2}Y]\bm\lambda_2 + \nonumber\\
&& \bm\lambda_2^\top [YK_{n2}^\top
M_2^{-1}\bar{K}_2\tilde{K}_1^{-1}K_{n1}Y]\bm\lambda_1 +
\bm\lambda_2^\top [YK_{n2}^\top M_2^{-1}K_{n2}Y]\bm\lambda_2 \Big\}
-\textbf{1}^\top
(\boldsymbol{\lambda}_1+\boldsymbol{\lambda}_2)\nonumber\\
&=&\frac{1}{2}(\boldsymbol{\lambda}_1^\top
\;\boldsymbol{\lambda}_2^\top) {\left(
\begin{array}{cc}
A&\; B\\
B^\top&\; D
\end{array}
\right)} {\left(
\begin{array}{c}
\boldsymbol{\lambda}_1\\
\boldsymbol{\lambda}_2
\end{array}
\right)} - {\left(
\begin{array}{c}
\boldsymbol{\lambda}_1\\
\boldsymbol{\lambda}_2
\end{array}
\right)^\top \textbf{1}_{2n}}, \nonumber
\end{eqnarray}
where
\begin{equation}
A\triangleq YK_{n1}^\top M_1^{-1}K_{n1}Y, \quad
 B\triangleq
YK_{n1}^\top M_1^{-1}\bar{K}_1\tilde{K}_2^{-1}K_{n2}Y, \quad
D\triangleq YK_{n2}^\top M_2^{-1}K_{n2}Y , \label{eqn07031S}
\end{equation}
$\textbf{1}_{2n}=(1,\ldots,1_{(2n)})^\top$, and we have used the
fact that
\begin{equation}
YK_{n1}^\top M_1^{-1}\bar{K}_1\tilde{K}_2^{-1}K_{n2}Y =[YK_{n2}^\top
M_2^{-1}\bar{K}_2\tilde{K}_1^{-1}K_{n1}Y]^\top. \nonumber
\end{equation}
Because of the convexity of function $-{g}$, we affirm that matrix
$\left(
\begin{array}{cc}
A&\; B\\
B^\top&\; D
\end{array}
\right)$ is positive semidefinite.

Hence, the optimization problem in (\ref{eqn05171S}) can be
rewritten as
\begin{eqnarray}
\min_{\boldsymbol{\lambda}_1, \boldsymbol{\lambda}_2}
&&\frac{1}{2}(\boldsymbol{\lambda}_1^\top
\;\boldsymbol{\lambda}_2^\top) {\left(
\begin{array}{cc}
A&\; B\\
B^\top&\; D
\end{array}
\right)} {\left(
\begin{array}{c}
\boldsymbol{\lambda}_1\\
\boldsymbol{\lambda}_2
\end{array}
\right)} - {\left(
\begin{array}{c}
\boldsymbol{\lambda}_1\\
\boldsymbol{\lambda}_2
\end{array}
\right)^\top} \mathbf{1}_{2n} \nonumber\\
\mbox{s.t.} && \left\{
\begin{array}{l}
0\preceq \boldsymbol{\lambda}_{1}\preceq C_1 \textbf{1}, \\
0\preceq \boldsymbol{\lambda}_{2}\preceq C_1 \textbf{1}.
\end{array}
\right. \label{eqn07032S}
\end{eqnarray}

After solving this problem, we can then obtain classifier parameters
$\boldsymbol{\alpha}_{1}$ and $\boldsymbol{\alpha}_{2}$ using
(\ref{eqnalpha1S}) and (\ref{eqnalpha2S}), which are finally used by
(\ref{eqn1310111S}).

\vskip 0.2in
\bibliography{MvPBbib}
\end{document}